\documentclass{article}

\usepackage{microtype}
\usepackage{graphicx}
\usepackage{subfigure}
\usepackage{booktabs} 
\usepackage{dsfont}
\usepackage{amsthm}
\usepackage{thmtools,thm-restate}
\usepackage[toc,page]{appendix}
\usepackage[none]{hyphenat}
\usepackage{comment}

\usepackage[accepted]{icml2018}

\usepackage{amsfonts}
\usepackage[fleqn]{amsmath}

\usepackage{bbold}
\usepackage{algorithm}
\usepackage{algorithmic}[1]
\usepackage{graphicx}
\usepackage{subfigure} 
\newcommand{\norm}[1]{\left\lVert#1\right\rVert}

\newcommand{\That}[2]{\widehat{T}^{#1}(#2)}
\newcommand{\TG}[2]{T^{#1}_{\mathcal G}(#2)}
\newcommand{\ThatG}[2]{\widehat{T}^{#1}_{\mathcal G}(#2)}
\newcommand{\ThatGn}[0]{\widehat{T}_{\mathcal G}}
\newcommand{\Thatn}[0]{\widehat{T}}
\newcommand{\TGn}[0]{T_{\mathcal G}}

\newcommand{\fancyS}[0]{\mathcal{S}}
\newcommand{\fancyA}[0]{\mathcal{A}}
\newcommand{\fancyG}[0]{\mathcal{G}}

\newcommand{\real}[0]{\mathbb R}
\newcommand{\ReLu}{\mbox{ReLu}}

\newtheorem{theorem}{Theorem}
\newtheorem{defn}{Definition}

\newtheorem{corollary}{Corollary}
\newtheorem{claim}{Claim}
\icmltitlerunning{Lipschitz Continuity in Model-based Reinforcement Learning}
\begin{document}

\twocolumn[
\icmltitle{Lipschitz Continuity in Model-based Reinforcement Learning}
\icmlsetsymbol{equal}{*}

\begin{icmlauthorlist}
\icmlauthor{Kavosh Asadi}{equal,br}
\icmlauthor{Dipendra Misra}{equal,cor}
\icmlauthor{Michael L. Littman}{br}
\end{icmlauthorlist}

\icmlaffiliation{br}{Department of Computer Science, Brown University, Providence, USA}
\icmlaffiliation{cor}{Department of Computer Science and Cornell Tech, Cornell University, New York, USA}

\icmlcorrespondingauthor{Kavosh Asadi}{kavosh@brown.edu}

\icmlkeywords{Machine Learning, ICML}

\vskip 0.3in
]
\printAffiliationsAndNotice{\icmlEqualContribution} 
\begin{abstract}
We examine the impact of learning Lipschitz continuous models in the context of model-based reinforcement learning. We provide a novel bound on multi-step prediction error of Lipschitz models where we quantify the error using the Wasserstein metric. We go on to prove an error bound for the value-function estimate arising from Lipschitz models and show that the estimated value function is itself Lipschitz. We conclude with empirical results that show the benefits of controlling the Lipschitz constant of neural-network models.
\end{abstract}	

\section{Introduction}
The model-based approach to reinforcement learning (RL) focuses on predicting the dynamics of the environment to plan and make high-quality decisions~\cite{kaelbling1996reinforcement,sutton98}. Although the behavior of model-based algorithms in tabular environments is well understood and can be effective~\citep{sutton98}, scaling up to the approximate setting can cause instabilities. Even small model errors can be magnified by the planning process resulting in poor performance~\cite{talvitie2014model}.

In this paper, we study model-based RL through the lens of Lipschitz continuity, intuitively related to the smoothness of a function. We show that the ability of a model to make accurate multi-step predictions is related to the model's one-step accuracy, but also to the magnitude of the  Lipschitz constant (smoothness) of the model. We further show that the dependence on the Lipschitz constant carries over to the value-prediction problem, ultimately influencing the quality of the policy found by planning. 

We consider a setting with continuous state spaces and stochastic transitions where we quantify the distance between distributions using the Wasserstein metric. We introduce a novel characterization of models, referred to as a Lipschitz model class, that represents stochastic dynamics using a set of component deterministic functions. This allows us to study any stochastic dynamic using the Lipschitz continuity of its component deterministic functions. To learn a Lipschitz model class in continuous state spaces, we provide an Expectation-Maximization algorithm~\citep{dempster1977maximum}.

One promising direction for mitigating the effects of inaccurate models is the idea of limiting the complexity of the learned models or reducing the horizon of planning~\cite{jiang15}. Doing so can sometimes make models more useful, much as regularization in supervised learning can improve generalization performance \citep{tibshirani1996regression}. In this work, we also examine a type of regularization that comes from controlling the Lipschitz constant of models. This regularization technique can be applied efficiently, as we will show, when we represent the transition model by neural networks.

\section{Background}

We consider the Markov decision process (MDP) setting in which the RL problem is formulated by the tuple $\langle \mathcal{S},\mathcal{A},R,T,\gamma
\rangle$. Here, by $\mathcal{S}$ we mean a continuous state space and by $\mathcal{A}$ we mean a discrete action set. The functions $R:S\times A\rightarrow\mathbb R$ and
$\ T\mathcal{:S\times\ A \rightarrow}\  \textrm{Pr}(\fancyS)$ denote the reward and transition dynamics. Finally, $\gamma \in [0,1)$ is the
discount rate. If $|\mathcal{A}|=1$, the setting is called a Markov reward process (MRP).
\subsection{Lipschitz Continuity}
Our analyses leverage the ``smoothness'' of various functions, quantified as follows.\begin{defn}
Given two metric spaces $(M_1,d_1)$ and $(M_2,d_2)$ consisting of a space and a distance metric, a function $f:M_1\mapsto M_2$ is \emph{Lipschitz continuous} (sometimes simply Lipschitz) if the Lipschitz constant, defined as
\begin{equation}
	K_{d_1,d_2}(f):=\sup_{s_1\in M_1,s_2\in M_1}\frac{d_2\big(f(s_1),f(s_2)\big)}{d_1(s_1,s_2)}\ ,
\end{equation} 
is finite.
\end{defn}

Equivalently, for a Lipschitz $f$, $$\forall s_1,\forall s_2 \quad d_2\big(f(s_1),f(s_2)\big)\leq K_{d_1,d_2}(f)\ d_1(s_1,s_2) \ \ . $$

The concept of Lipschitz continuity is visualized in Figure~\ref{Lipschitz_illustration}.

A Lipschitz function $f$ is called a \emph{non-expansion} when $K_{d_1,d_2}(f)=1$ and a \emph{contraction} when $K_{d_1,d_2}(f)< 1$. Lipschitz continuity, in one form or another, has been a key tool in the theory of reinforcement learning~\citep{bertsekas1975convergence,bertsekas1995neuro,littman1996generalized,muller1996optimal,ferns2004metrics,hinderer2005lipschitz,RachelsonL10,szepesvari2010algorithms,pazis2013pac,pirotta2015policy,pires2016policy,berkenkamp2017safe,bellemare2017distributional} and bandits~\citep{kleinberg2008multi,bubeck2011x}. Below, we also define Lipschitz continuity over a subset of inputs.

\begin{figure}
\centering
                \includegraphics[width=.4\columnwidth]{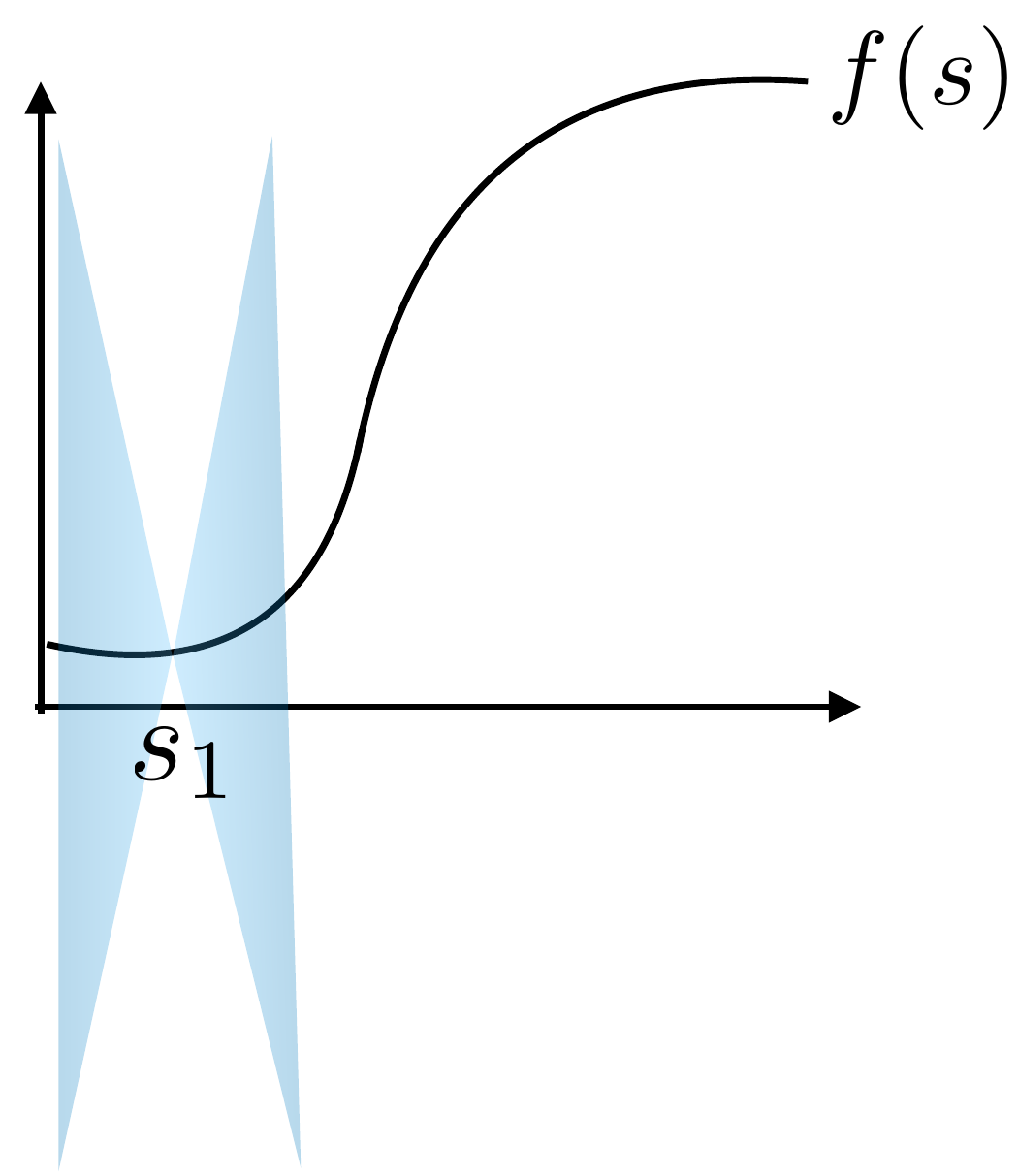}
        \caption{An illustration of Lipschitz continuity. Pictorially, Lipschitz continuity ensures that $f$ lies in between the two affine functions (colored in blue) with slopes $K$ and $-K$.}
        \label{Lipschitz_illustration}
\end{figure}

\begin{defn}
A function $f:M_1 \times \fancyA \mapsto M_2$ is \emph{uniformly Lipschitz continuous} in $\fancyA$ if  \begin{equation}
	K_{d_1,d_2}^\fancyA(f):=\sup_{a\in \fancyA}\sup_{s_1,s_2}\frac{d_2\big(f(s_1,a),f(s_2,a)\big)}{d_1(s_1,s_2)}\ ,
\end{equation}
is finite.
\end{defn}	
Note that the metric $d_1$ is defined only on $M_1$.
\subsection{Wasserstein Metric}
We quantify the distance between two distributions using the following metric:
\begin{defn}
Given a metric space $(M,d)$ and the set $\mathbb P(M)$ of all probability measures on $M$, the \emph{Wasserstein metric} (or the 1st Kantorovic metric) between two probability distributions $\mu_1$ and $\mu_2$ in $\mathbb P(M)$ is defined as
\begin{equation}
	W(\mu_1,\mu_2):=\inf_{j \in \Lambda}\int\!\int j(s_1,s_2) d(s_1,s_2)ds_2\ ds_1 \ ,
	\label{wasserstein_primal}
\end{equation}
where $\Lambda$ denotes the collection of all joint distributions $j$ on $M \times M$ with marginals $\mu_1$ and $\mu_2$~\citep{vaserstein1969markov}.
\end{defn}
Sometimes referred to as ``Earth Mover's distance'', Wasserstein is the minimum expected distance between pairs of points where the joint distribution $j$ is constrained to match the marginals $\mu_1$ and $\mu_2$. New applications of this metric are discovered in machine learning, namely in the context of generative adversarial networks~\citep{arjovsky2017wasserstein} and value distributions in reinforcement learning~\citep{bellemare2017distributional}.

Wasserstein is linked to Lipschitz continuity using duality:
\begin{equation}
	W(\mu_1,\mu_2)=\sup_{f:K_{d,d_{\real}}(f)\leq 1}\int \big(f(s)\mu_{1}(s)-f(s)\mu_{2}(s)\big)ds\ .
\label{KR_duality}
\end{equation}
This equivalence, known as Kantorovich-Rubinstein duality \citep{villani2008optimal}, lets us compute Wasserstein by maximizing over a Lipschitz set of functions $f:\mathcal{S}\mapsto \real$, a relatively easier problem to solve. In our theory, we utilize both definitions, namely the primal definition (\ref{wasserstein_primal}) and the dual definition (\ref{KR_duality}).
\section{Lipschitz Model Class}
We introduce a novel representation of stochastic MDP transitions in terms of a distribution over a set of deterministic components.
\begin{defn}Given a metric state space $(\fancyS,d_{\fancyS})$ and an action space $\fancyA$, we define $F_g$ as a collection of functions: $F_g=\{f:\fancyS\mapsto \fancyS\}$ distributed according to $g(f\mid a)$ where $a \in \fancyA$. We say that $F_g$ is \emph{a Lipschitz model class} if $$K_{F}:=\sup_{f\in F_g}\ K_{d_{\fancyS},d_{\fancyS}}(f)\ ,$$
is finite.
\label{Lipschitz_model_class}
\end{defn}

Our definition captures a subset of stochastic transitions, namely ones that can be represented as a state-independent distribution over deterministic transitions. An example is provided in Figure~\ref{russel_norvig_grid}. We further prove in the appendix (see Claim 1) that any finite MDP transition probabilities can be decomposed into a state-independent distribution $g$ over a finite set of deterministic functions $f$.
\newcommand{\ack}[1]{^{\mbox{\footnotesize{\em #1}}}}
\newcommand{\ac}[1]{{\mbox{{\em #1}}}}

\begin{figure}
\centering
            \includegraphics[width=0.6\columnwidth]{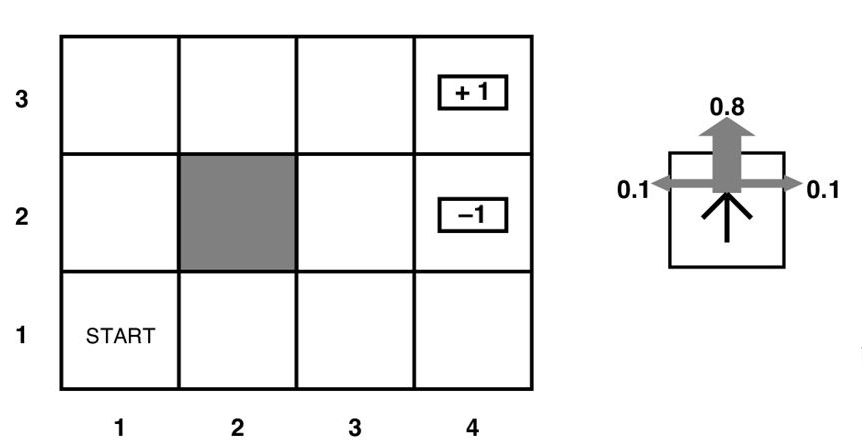}
       	        \caption{An example of a Lipschitz model class in a gridworld environment~\cite{russel1995artifial}. The dynamics are such that any action choice results in an attempted transition in the corresponding direction with probability 0.8 and in the neighboring directions with probabilities 0.1 and 0.1. We can define $F_{g}=\{f\ack{up},f\ack{right},f\ack{down},f\ack{left}\}$ where each $f$ outputs a deterministic next position in the grid (factoring in obstacles). For $a=\ac{up}$, we have: $g(f\ack{up}\mid a=\ac{up})=0.8,\ g(f\ack{right}\mid a=\ac{up})=g(f\ack{left}\mid a=\ac{up})=0.1$, and $g(f\ack{down}\mid a=\ac{up})=0$. Defining distances between states as their Manhattan distance in the grid, then $\forall f \sup_{s_1,s_2}\big(d(f(s_1),f(s_2)\big)/d(s_1,s_2)=2$, and so $K_F=2$. So, the four functions and $g$ comprise a Lipschitz model class.}
        \label{russel_norvig_grid}
\end{figure}

Associated with a Lipschitz model class is a transition function given by:
\begin{eqnarray*}
\widehat T(s'\mid s,a)=\sum_{f}\mathds{1}\big(f(s)=s'\big)\ g(f\mid a)\ .\\
\end{eqnarray*}
Given a state distribution $\mu(s)$, we also define a generalized notion of transition function $\ThatGn(\cdot\mid \mu, a)$ given by:
\begin{eqnarray*}
\ThatG{}{s'\mid \mu,a}=\int_{s}\underbrace{\sum_{f}\mathds{1}\big(f(s)=s'\big)\ g(f\mid a)}_{\widehat T(s'\mid s,a)}\mu(s)ds\ .
\end{eqnarray*}

We are primarily interested in $K^{\fancyA}_{d,d}(\ThatGn)$, the Lipschitz constant of $\ThatGn$.
However, since $\ThatGn$ takes as input a probability distribution and also outputs a probability distribution, we require a notion of distance between two distributions. This notion is quantified using Wasserstein and is justified in the next section.
\section{On the Choice of Probability Metric}
We consider the stochastic model-based setting and show through an example that the Wasserstein metric is a reasonable choice compared to other common options.

Consider a uniform distribution over states $\mu(s)$ as shown in black in Figure~\ref{wasserstein_is_good} (top). Take a transition function $\TGn$ in the environment that, given an action $a$, uniformly randomly adds or subtracts a scalar $c_1$. The distribution of states after one transition is shown in red in  Figure~\ref{wasserstein_is_good} (middle). Now, consider a transition model $\ThatGn$ that approximates $\TGn$ by uniformly randomly adding or subtracting the scalar $c_2$. The distribution over states after one transition using this imperfect model is shown in blue in Figure~\ref{wasserstein_is_good} (bottom).
\begin{figure}
\centering
                \includegraphics[width=0.55\columnwidth]{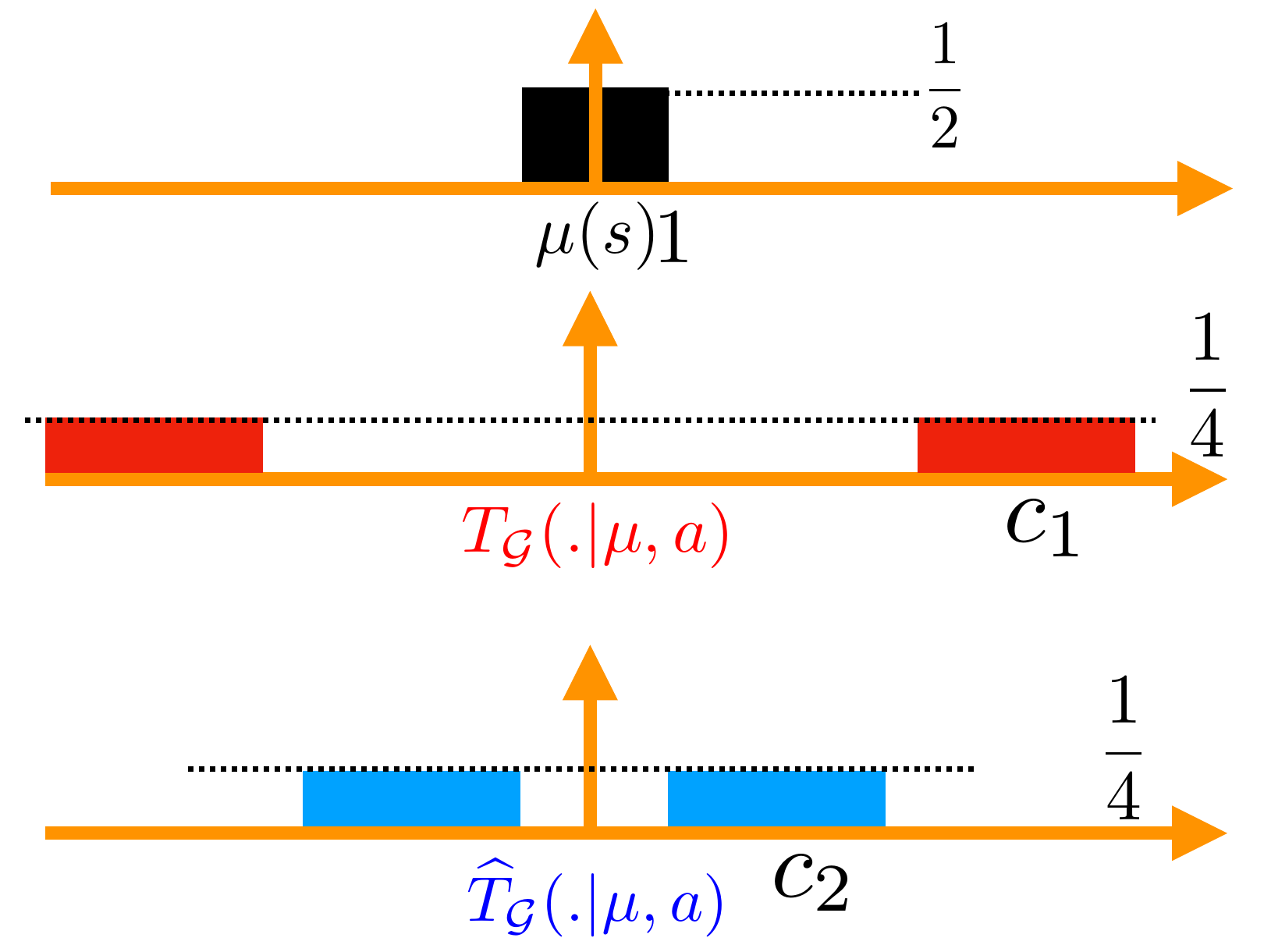}
       	        \caption{A state distribution $\mu(s)$ (top), a stochastic environment that randomly adds or subtracts $c_1$ (middle), and an approximate transition model that randomly adds or subtracts a second scalar $c_2$ (bottom).}
        \label{wasserstein_is_good}
\end{figure}
We desire a metric that captures the similarity between the output of the two transition functions. We first consider Kullback-Leibler (KL) divergence and observe that:
\begin{eqnarray*}
    &&KL\big(\TGn(\cdot \mid \mu,a),\ThatGn(\cdot \mid \mu,a)\big)\\
    &&:=\int \TGn(s'\mid\mu,a) \log{\frac{\TGn(s'\mid\mu,a)}{ \ThatGn(s'\mid\mu,a)}ds'}=\infty\ ,
\end{eqnarray*}
unless the two constants are exactly the same.

The next possible choice is Total Variation (TV) defined as:
\begin{eqnarray*}
    &&TV\big(\TG{}{\cdot \mid \mu,a},\ThatG{}{\cdot \mid \mu,a}\big)\\
    &&:=\frac{1}{2}\int\big|\TG{}{s'\mid\mu,a}-\ThatG{}{s'\mid\mu,a}\big| ds'=1\ ,
\end{eqnarray*}
if the two distributions have disjoint supports regardless of how far the supports are from each other.

In contrast, Wasserstein is sensitive to how far the constants are as:
$$W\big(\TGn(\cdot\mid\mu,a),\ThatGn(\cdot\mid\mu,a)\big)=|c_1-c_2|\ .$$
It is clear that, of the three, Wasserstein corresponds best to the intuitive sense of how closely $\TGn$ approximates $\ThatGn$. This is particularly important in high-dimensional spaces where the true distribution is known to usually lie in low-dimensional manifolds. \citep{narayanan2010sample}
\section{Understanding the Compounding Error Phenomenon}

To extract a prediction with a horizon $n>1$, model-based algorithms typically apply the model for $n$ steps by taking the state input in step $t$ to be the state output from the step $t-1$. Previous work has shown that model error can result in poor long-horizon predictions and ineffective planning~\citep{talvitie2014model,talvitie2017self}. Observed even beyond reinforcement learning~\citep{lorenz1972predictability,venkatraman2015improving}, this is referred to as the \emph{compounding error phenomenon}. The goal of this section is to provide a bound on multi-step prediction error of a model. We formalize the notion of model accuracy below: \begin{defn}
Given an MDP with a transition function $T$, we identify a Lipschitz model $F_g$ as \textit{$\Delta$-accurate} if its induced $\widehat{T}$ satisfies:
\begin{equation*}\forall s \ \forall a\quad W\big(\widehat T(\cdot\mid s,a),T(\cdot\mid s,a)\big)\leq \Delta\ .\end{equation*}\label{delta_accurate_model}
\end{defn}We want to express the multi-step Wasserstein error in terms of the single-step Wasserstein error and the Lipschitz constant of the transition function $\ThatGn$. We provide a bound on the Lipschitz constant of $\ThatGn$ using the following lemma:
\begin{restatable}{lemma}{primelemma}
	\label{lemma1}
	A generalized transition function $\ThatGn$ induced by a Lipschitz model class $F_g$ is Lipschitz with a constant:
	\begin{eqnarray*}
	K^{\fancyA}_{W,W}(\ThatGn):=\sup_{a}\!\sup_{\mu_1,\mu_2}\!\frac{W\big(\ThatGn(\cdot|\mu_1,a),\ThatGn(\cdot|\mu_2,a)\big)}{W(\mu_1,\mu_2)}\!\leq\! K_{F}
	\end{eqnarray*}
	\label{Lipschit_model_class_lemma}
\end{restatable}
Intuitively, Lemma \ref{lemma1} states that, if the two input distributions are similar, then for any action the output distributions given by $\ThatGn$ are also similar up to a $K_F$ factor. We prove this lemma, as well as the subsequent lemmas, in the appendix.

Given the one-step error (Definition \ref{delta_accurate_model}), a start state distribution $\mu$ and a fixed sequence of actions $a_0,...,a_{n-1}$, we desire a bound on $n$-step error:
$$\delta(n):=W\big(\ThatGn^n(\cdot\mid\mu),\TGn^n(\cdot\mid\mu)\big)\ ,$$
where $\ThatGn^n(\cdot|\mu):=\underbrace{\ThatG{}{\cdot|\ThatG{}{\cdot|...\ThatG{}{\cdot|\mu,a_0}...,a_{n-2}},a_{n-1}}}_{n\ \textrm{recursive calls}}$ and $\TGn^n(\cdot\mid\mu)$ is defined similarly. We provide a useful lemma followed by the theorem.
\begin{restatable}{lemma}{compositionlemma}
(Composition Lemma) Define three metric spaces $(M_1,d_1)$, $(M_2,d_2)$, and $(M_3,d_3)$. Define Lipschitz functions $f:M_2\mapsto M_3$ and $g:M_1\mapsto M_2$ with constants $K_{d_2,d_3}(f)$ and $K_{d_1,d_2}(g)$. Then, $h:f \circ g:M_1 \mapsto M_3$ is Lipschitz with constant $K_{d_1,d_3}(h)\leq K_{d_2,d_3}(f) K_{d_1,d_2}(g)$.
\label{lemma_composition}
\end{restatable}
Similar to composition, we can show that summation preserves Lipschitz continuity with a constant bounded by the sum of the Lipschitz constants of the two functions. We omitted this result due to brevity.
\begin{theorem}
\label{theorem_compound}
Define a $\Delta$-accurate $\ThatGn$ with the Lipschitz constant $K_F$ and an MDP with a Lipschitz transition function $\TGn$ with constant $K_{T}$. Let $\bar K=\min\{K_{F},K_{T}\}$. Then $\forall n\geq 1$:
$$\delta(n):=W\big(\ThatG{n}{\cdot\mid\mu},T^{n}_{\fancyG}(\cdot\mid\mu)\big)\leq \Delta\sum_{i=0}^{n-1} (\bar K)^{i} \ .$$
\end{theorem}
\begin{proof}
We construct a proof by induction. Using Kantarovich-Rubinstein duality (Lipschitz property of $f$ not shown for brevity) we first prove the base of induction:

$\begin{aligned}
&\delta(1):=W\big(\ThatG{}{\cdot\mid\mu,a_0},T^{}_{\fancyG}(\cdot\mid\mu,a_0)\big) \\
&:=\sup_{f}\!\!\int \!\int\!\! \big(\That{}{s'\mid s,a_0}\!-\!T(s'\mid s,a_0)\big)f(s')\mu(s)\ ds\ ds' \\
&\leq \int\!\!\underbrace{\sup_{f} \!\!\int\!\!\big(\That{}{s'|s,a_0}\!-\!T(s'|s,a_0)\big)f(s')\ ds'}_{=W\big(\That{}{\cdot\mid s,a_0},T(\cdot\mid s,a_0)\big) \textrm{due to duality (\ref{KR_duality})} }\mu(s)\ ds \\
&=\int\underbrace{W\big(\That{}{\cdot\mid s,a_0},T(\cdot\mid s,a_0)\big)}_{\leq\Delta\ \textrm{due to Definition \ref{delta_accurate_model}}}\mu(s)\ ds \\
&\leq\int\Delta\ \mu(s)\ ds=\Delta \ .
\end{aligned}$

We now prove the inductive step. Assuming $\delta(n-1):=W\big(\ThatG{n-1}{\cdot\mid\mu},T^{n-1}_{\fancyG}(\cdot\mid\mu)\big)\leq \Delta\sum_{i=0}^{n-2} (K_F)^i$ we can write:

$\begin{aligned}
&\delta(n):=W\big(\ThatG{n}{\cdot\mid\mu},T^{n}_{\fancyG}(\cdot\mid\mu)\big)\nonumber\\
&\leq  	W\Big(\ThatG{n}{\cdot\mid\mu},\ThatGn \big(\cdot \mid\TG{n-1}{\cdot\mid\mu},a_{n-1}\big)\Big) \nonumber\\
&+\!W\Big(\ThatGn \big(\cdot\mid\TG{n-1}{\cdot\mid\mu},a_{n-1}\big),T^{n}_{\fancyG}(\cdot\mid\mu)\! \Big)\ \textrm{(Triangle ineq)}\\
&=\!W\Big(\!\ThatG{}{\cdot\mid\ThatG{n-1}{\cdot\mid\mu},a_{n-1}},\ThatGn \big(\cdot\mid\TG{n-1}{\cdot\mid\mu},a_{n-1}\big)\!\Big) \textrm{}\nonumber\\
&\!+\!W\Big(\!\ThatGn\!\big(\cdot\mid\TG{n-1}{\cdot\mid\mu},a_{n-1}\big),T^{}_{\fancyG}(\cdot\mid\TG{n-1}{\cdot\mid\mu},a_{n-1}) \Big)\nonumber
\end{aligned}$

We now use Lemma~\ref{Lipschit_model_class_lemma} and Definition~\ref{delta_accurate_model} to upper bound the first and the second term of the last line respectively.
\begin{eqnarray}
\delta(n)&\leq& K_F\ W\big(\ThatG{n-1}{\cdot\mid\mu},T^{n-1}_{\fancyG}(\cdot\mid\mu)\big)+\Delta\nonumber\\
&=& K_F\ \delta(n-1)+\Delta\le \Delta \sum_{i=0}^{n-1}(K_F)^i\label{first_inequality} \ .
\end{eqnarray}
Note that in the triangle inequality, we may replace $\ThatGn \big(\cdot\mid~\TG{n-1}{\cdot\mid\mu}\big)$ with $\TGn \big(\cdot\mid\ThatG{n-1}{\cdot\mid\mu}\big)$ and follow the same basic steps to get:
\begin{equation}
    W\big(\ThatG{n}{\cdot\mid\mu},T^{n}_{\fancyG}(\cdot\mid\mu)\big) \leq \Delta \sum_{i=0}^{n-1}(K_T)^i\ . \label{second_inequality}
\end{equation}
Combining (\ref{first_inequality}) and~(\ref{second_inequality}) allows us to write:
\begin{eqnarray*}
    \delta(n)&=&W\big(\ThatG{n}{\cdot\mid \mu},T^{n}_{\fancyG}(\cdot\mid \mu)\big)\nonumber\\
    &\leq& \min\left\{\Delta \sum_{i=0}^{n-1}(K_T)^i,\Delta \sum_{i=0}^{n-1}(K_F)^i\right\}\nonumber\\
	&=& \Delta \sum_{i=0}^{n-1}(\bar K)^i\ ,
\end{eqnarray*}
which concludes the proof.
\end{proof}

There exist similar results in the literature relating one-step transition error to multi-step transition error and sub-optimality bounds for planning with an approximate model. The Simulation Lemma~\cite{kearns2002near,strehl2009reinforcement} is for discrete state MDPs and relates error in the one-step model to the value obtained by using it for planning.  A related result for continuous state-spaces~\citep{kakade2003exploration} bounds the error in estimating the probability of a trajectory using total variation. A second related result~\citep{venkatraman2015improving} provides a slightly looser bound for prediction error in the deterministic case---our result can be thought of as a generalization of their result to the probabilistic case.
\section{Value Error with Lipschitz Models}

We next investigate the error in the state-value function induced by a Lipschitz model class. 
To answer this question, we consider an MRP $M_1$ denoted by $\langle \fancyS,\fancyA, T,R,\gamma\rangle$ and a second MRP $M_2$ that only differs from the first in its transition function $\langle\fancyS, \fancyA, \widehat{T},R,\gamma\rangle$. Let $\fancyA = \{a\}$ be the action set with a single action $a$. We further assume that the reward function is only dependent upon state.
We first express the state-value function for a start state $s$ with respect to the two transition functions. By $\delta_s$ below, we mean a Dirac delta function denoting a distribution with probability 1 at state $s$.
\begin{eqnarray*}
&&V_T(s) := \sum_{n=0}^\infty \gamma^n \int \TGn^n(s'|\delta_s)R(s')\ ds'\ ,
\end{eqnarray*}
\begin{eqnarray*}
&&V_{\widehat T}(s) := \sum_{n=0}^\infty \gamma^n \int \ThatGn^n(s'|\delta_s)R(s')\ ds'\ .
\end{eqnarray*}

Next we derive a bound on $\big|V_T(s)-V_{\widehat T}(s)\big|\ \forall s$.

\begin{theorem}
Assume a Lipschitz model class $F_g$ with a $\Delta$-accurate $\widehat{T}$ with $\bar{K} = \min\{K_F, K_T\}$. Further, assume a Lipschitz reward function with constant $K_R=K_{d_\fancyS,\real}(R)$. Then $\forall s\in \fancyS$ and $\bar{K} \in [0, \frac{1}{\gamma})$
$$\big|V_T(s)-V_{\widehat T}(s)\big| \le \frac{\gamma K_R\Delta}{(1-\gamma)(1-\gamma \bar{K})}\ .$$
\label{theorem:value_error}
\end{theorem}
\begin{proof}
We first define the function 
$f(s) = \frac{R(s)}{K_R}\ .$ It can be observed that $K_{d_\fancyS,\real}(f)=1$. We now write:
\begin{eqnarray*}
    && V_T(s)-V_{\widehat T}(s)\\
    &&=\sum_{n=0}^{\infty}\gamma^n\int R(s')\big(\TGn^{n}(s'\mid \delta_s)-\ThatGn^{n}(s'\mid\delta_s)\big)\ ds'\\
    &&= K_R\sum_{n=0}^{\infty}\gamma^n\int f(s')\big(\TGn^{n}(s'\mid\delta_s)-\ThatGn^{n}(s'\mid\delta_s)\big)\ ds'
\end{eqnarray*}
Let $\mathcal{F} = \{h: K_{d_\mathcal{S}, \mathbb{R}}(h) \le 1\}$. Then given $f \in \mathcal{F}$:
\begin{eqnarray*}
&&\!\!\!\!\!\!\!\!\!\!\!\!\!\!K_R\sum_{n=0}^{\infty}\gamma^n\int f(s')\big(\TGn^{n}(s'|\delta_s)-\ThatGn^{n}(s'|\delta_s)\big)ds'\\
&\le& K_R\sum_{n=0}^{\infty}\gamma^n\underbrace{\sup_{f\in \mathcal{F}}\int f(s')\big(\TGn^{n}(s'\mid\delta_s)-\ThatGn^{n}(s'\mid\delta_s)\big)ds'}_{:=W\big(\TGn^n(.\mid \delta_s),\ThatGn^{n}(.\mid \delta_s)\big) \textrm{due to duality (\ref{KR_duality})}}\\
&=& K_R\sum_{n=0}^{\infty}\gamma^n \underbrace{W\big(\TGn^n(.\mid \delta_s),\ThatGn^{n}(.\mid \delta_s)\big)}_{\leq \sum_{i=0}^{n-1}\Delta (\bar{K})^i \textrm{ due to Theorem \ref{theorem_compound}}} \ \\
&\leq& K_R\sum_{n=0}^{\infty}\gamma^n \sum_{i=0}^{n-1}\Delta (\bar{K})^i\\
&=&K_R\Delta\sum_{n=0}^{\infty}\gamma^n \frac{1-\bar{K}^n}{1-\bar{K}}\\
&=&\frac{\gamma K_R\Delta}{(1-\gamma)(1-\gamma \bar{K})}\ .
\end{eqnarray*}
We can derive the same bound for $V_{\widehat T}(s)-V_T(s)$ using the fact that Wasserstein distance is a metric, and therefore symmetric, thereby completing the proof.
\end{proof}
Regarding the tightness of our bounds, we can show that when the transition model is deterministic and linear then  Theorem \ref{theorem_compound} provides a tight bound. Moreover, if the reward function is linear, the bound provided by Theorem \ref{theorem:value_error} is tight. (See Claim 2 in the appendix.) Notice also that our proof does not require a bounded reward function.
\section{Lipschitz Generalized Value Iteration}
We next show that, given a Lipschitz transition model, solving for the fixed point of a class of Bellman equations yields a Lipschitz state-action value function. Our proof is in the context of Generalized Value Iteration (GVI) \citep{littman1996generalized}, which defines Value Iteration \citep{bellman1957markovian} for planning with arbitrary backup operators.

\begin{algorithm}
\caption{GVI algorithm}
\begin{algorithmic}
   \STATE {\bfseries Input:} initial $\widehat Q(s,a)$, $\delta$, and choose an operator $f$
   \REPEAT
   \FOR{each $s,a\in \mathcal{S}\times \mathcal{A}$}
   \STATE $\widehat Q(s,a)\! \leftarrow\! R(s,a)\!+\!\gamma\!\int\widehat T(s'\mid s,a) f\big(\widehat Q(s',\cdot)\big)ds'$ 
   \ENDFOR
   \UNTIL{ \textrm{convergence}}
\end{algorithmic}
\label{GVI}
\end{algorithm}
To prove the result, we make use of the following lemmas.
\begin{restatable}{lemma}{mullerLemma}
	Given a Lipschitz function $f:\fancyS\mapsto \mathbb R$ with constant $K_{d_{\fancyS},d_{\real}}(f)$:
	\begin{equation*}K_{d_\fancyS,d_{\real}}^\fancyA \Big(\int\widehat T(s'|s,a)f(s')ds'\Big)\leq K_{d_{\fancyS},d_{\real}}(f) K^{\fancyA}_{d_\fancyS,W}\big(\widehat T\big)\ .\end{equation*}
\label{lipschitz_transition}
\end{restatable}
\vspace*{-\baselineskip}
\begin{restatable}{lemma}{LipschitzOperators}
	The following operators~\cite{mellowmax} are Lipschitz with constants:
	\begin{enumerate}
		\item $K_{\norm{}_{\infty},d_{R}}(\max(x))=K_{\norm{}_{\infty},d_{R}}\big(\textrm{mean}(x)\big)=K_{\norm{}_{\infty},d_{R}}(\epsilon$-$greedy(x))=1$ 
		\item $K_{\norm{}_{\infty},d_{R}}(mm_\beta(x):=\frac{\log\frac{\sum_{i}e^{\beta x_i}}{n}}{\beta})=1$
		\item $K_{\norm{}_{\infty},d_{R}}(boltz_{\beta}(x):=\frac{\sum_{i=1}^n x_i e^{\beta x_i}}{{\sum_{i=1}^n}e^{\beta}x_i})\leq\sqrt{|A|}+\beta V_{\max}|A|$
	\end{enumerate}
	\label{operators_Lipschitzness}
\end{restatable}
        \begin{theorem}
	For any non-expansion backup operator $f$ outlined in Lemma~\ref{operators_Lipschitzness}, GVI computes a value function with a Lipschitz constant bounded by $\frac{K^{\fancyA}_{d_{\fancyS},d_{R}}(R)}{1-\gamma K_{d_{\fancyS},W}( T)}\ $ if $\gamma K_{d_{\fancyS},W}^{\fancyA}( T)< 1$.
	\label{theorem_lipschitz_q}
\end{theorem}
\begin{proof}
From Algorithm~\ref{GVI}, in the $n$th round of GVI updates:
	$$\widehat Q_{n+1}(s,a) \leftarrow R(s,a)+\gamma\int T(s'\mid s,a) f\big(\widehat Q_{n}(s',\cdot)\big) ds'.$$
	Now observe that:
	\begin{eqnarray*}
	&&\!K^{\fancyA}_{d_{\fancyS},d_{R}}(\widehat Q_{n+1})\\
		&&\leq \! K^{\fancyA}_{d_{\fancyS},d_{R}}\!(R)\!+\!\gamma K_{d_{\fancyS},d_\real}^{\fancyA}\!\big(\!\int\!T(s'\mid s,a)\! f\big( \widehat Q_{n}(s',\cdot)\big)ds'\big) \\
		&&\leq K^{\fancyA}_{d_{\fancyS},d_{R}}(R)+\gamma K_{d_{\fancyS},W}^{\fancyA}( T)\ K_{d_{\fancyS,\real}}\Big(f\big( \widehat Q_{n}(s,\cdot)\big)\Big)\\
		&&\leq
		K^{\fancyA}_{d_{\fancyS},d_{R}}(R)+\gamma K_{d_{\fancyS},W}^{\fancyA}( T) K_{\norm{\cdot}_{\infty},d_{\real}}(f) K^{\fancyA}_{d_{\fancyS},d_\real}(\widehat Q_{n})\\
		&&=
		K^{\fancyA}_{d_{\fancyS},d_{R}}(R)+\gamma K_{d_{\fancyS},W}^{\fancyA}( T) K^{\fancyA}_{d_{\fancyS},d_\real}(\widehat Q_{n})
	\end{eqnarray*}
Where we used Lemmas \ref{lipschitz_transition}, \ref{lemma_composition}, and \ref{operators_Lipschitzness} for the second, third, and fourth inequality respectively. Equivalently:
\begin{eqnarray*}
K^{\fancyA}_{d_{\fancyS},d_{R}}(\widehat Q_{n+1})&\leq&
K^{\fancyA}_{d_{\fancyS},d_{\real}}(R)\!\sum_{i=0}^{n}\!\big(\gamma K_{d_{\fancyS},W}^{\fancyA}( T)\big)^i\\
&+&\big(\gamma K_{d_{\fancyS},W}^{\fancyA}(T)\big)^n \ K^{\fancyA}_{d_{\fancyS},d_{\real}}(\widehat Q_{0}) \ .
\end{eqnarray*}
By computing the limit of both sides, we get:
\begin{eqnarray*}\lim_{n\rightarrow\infty}K^{\fancyA}_{d_{\fancyS},d_{\real}}(\widehat Q_{n+1})\!\!\!\!&\!\leq\!\!\!&\!\!\!\lim_{n\rightarrow\infty}\! K^{\fancyA}_{d_{\fancyS},d_{\real}}(R)\!\sum_{i=0}^{n}\!\big(\gamma K_{d_{\fancyS},W}^{\fancyA}( T)\big)^i\\
&+&\lim_{n\rightarrow\infty}\big(\gamma K_{d_{\fancyS},W}^{\fancyA}(T)\big)^n \ K^{\fancyA}_{d_{\fancyS},d_{\real}}(\widehat Q_{0})\\
&=&\frac{K^{\fancyA}_{d_{\fancyS},d_{R}}(R)}{1-\gamma K_{d_{\fancyS},W}( T)} + 0 \ ,
\end{eqnarray*}
This concludes the proof.\end{proof}
Two implications of this result: First, PAC exploration in continuous state spaces is shown assuming a Lipschitz value function~\cite{pazis2013pac}. However, the theorem shows that it is sufficient to have a Lipschitz model, an assumption perhaps easier to confirm. The second implication relates to value-aware model learning (VAML) objective \citep{farahmand17a}. Using the above theorem, we can show that minimizing Wasserstein is equivalent to minimizing the VAML objective \citep{asadi2018wasserstein}.

\label{sectionLGVI}
\section{Experiments}
Our first goal in this section\footnote{We release the code here: github.com/kavosh8/Lip} is to compare TV, KL, and Wasserstein in terms of the ability to best quantify error of an imperfect model. To this end, we built finite MRPs with random transitions, $|\fancyS|=10$ states, and $\gamma=0.95$. In the first case the reward signal is randomly sampled from $[0,10]$, and in the second case the reward of an state is the index of that state, so small Euclidean norm between two states is an indication of similar values. For $10^5$ trials, we generated an MRP and a random model, and then computed model error and planning error~(Figure \ref{fig:random_MDPs}). We understand a good metric as the one that computes a model error with a high correlation with value error. We show these correlations for different values of $\gamma$ in Figure \ref{fig:random_MDPs_gamma}.

\begin{figure}[h]
\centering
        \begin{subfigure}
        \centering
                \includegraphics[width=0.49\columnwidth,height=95pt]{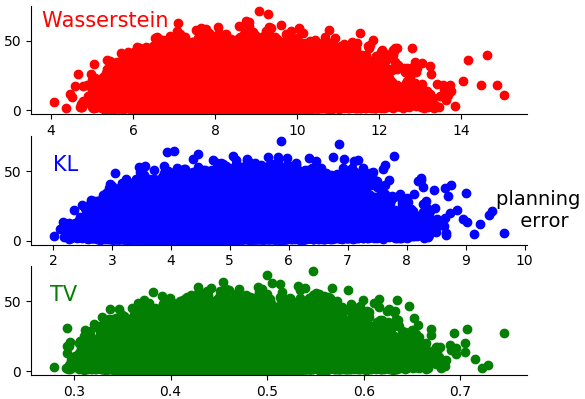}
        \end{subfigure}%
        \begin{subfigure}
        \centering
                \includegraphics[width=.49\columnwidth,height=95pt]{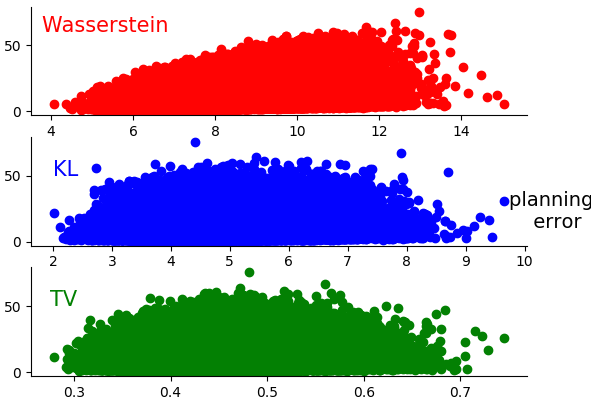}
        \end{subfigure}
        \caption{Value error (x axis) and model error (y axis).  When the reward is the index of the state (right), correlation between Wasserstein error and value-prediction error is high. This highlights the fact that when closeness in the state-space is an indication of similar values, Wasserstein can be a powerful metric for model-based RL. Note that Wasserstein provides no advantage given random rewards~(left).}
        \label{fig:random_MDPs}
\end{figure}
\begin{figure}[h]
\centering
        \begin{subfigure}
        \centering
                \includegraphics[width=0.475\columnwidth]{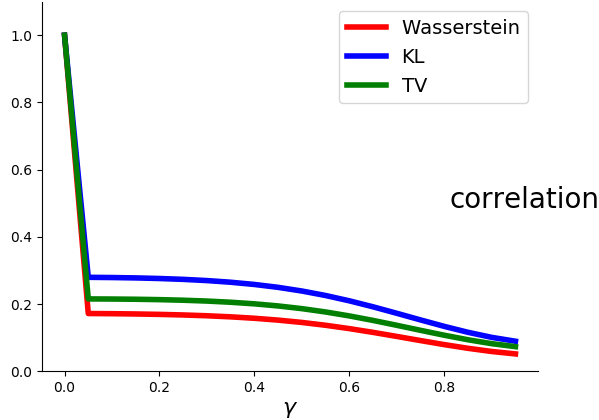}
        \end{subfigure}%
        \begin{subfigure}
        \centering
                \includegraphics[width=.475\columnwidth]{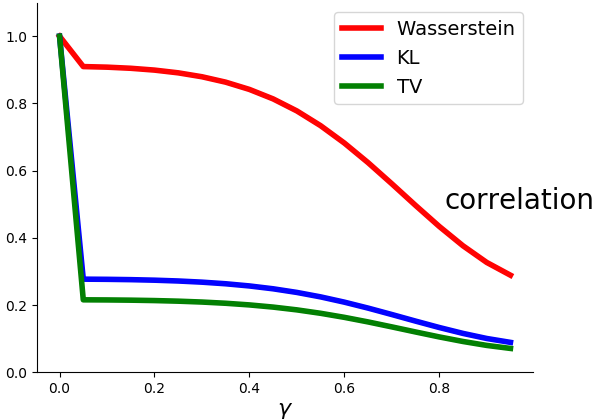}
        \end{subfigure}
        \caption{Correlation between value-prediction error and model error for the three metrics using random rewards (left) and index rewards (right). Given a useful notion of state similarities, low Wasserstein error is a better indication of planning error.}
        \label{fig:random_MDPs_gamma}
\end{figure}
\begin{table*}
    \centering
    \begin{tabular}{|c|c|c|c|c|}
    \hline
        Function $f$ & Definition & \multicolumn{3}{|c|}{Lipschitz constant $K_{\|\|_p, \|\|_p}(f)$}\\
        \hline
          &  & $p=1$ & $p=2$ & $p=\infty$\\
        \hline
         $\ReLu: R^n \rightarrow R^n$ & $\ReLu(x)_i := \max\{0 , x_i\}$ & 1 & 1 & 1\\
         $+b: R^n \rightarrow R^n, \forall b\in R^n$ & $+b(x) := x+ b$ & 1 & 1 & 1\\
         $\times W: R^n \rightarrow R^m, \,\,\forall W\in R^{m\times n}$ & $\times W(x) := Wx$ & $\sum_{j} \|W_j\|_\infty$ & $\sqrt{\sum_j \|W_j\|^2_2}$ & $\sup_j \|W_j\|_1$ \\
    \hline
    \end{tabular}
    \caption{Lipschitz constant for various functions used in a neural network. Here, $W_j$ denotes the $j$th row of a weight matrix $W$.}
    \label{tab:lipschitz-neural}
\end{table*}

It is known that controlling the Lipschitz constant of neural nets can help in terms of improving generalization error due to a lower bound on  Rademacher complexity~ \citep{Neyshabur15,bartlett2002rademacher}.  It then follows from Theorems~\ref{theorem_compound} and~\ref{theorem:value_error} that controlling the Lipschitz constant of a learned transition model can achieve better error bounds for multi-step and value predictions. To enforce this constraint during learning, we bound the Lipschitz constant of various operations used in building neural network. The bound on the constant of the entire neural network then follows from Lemma~\ref{lemma_composition}. In Table~\ref{tab:lipschitz-neural}, we provide Lipschitz constant for operations~(see Appendix for proof) used in our experiments. We quantify these results for different $p$-norms $\norm{\cdot}_{p}$.

Given these simple methods for enforcing Lipschitz continuity, we performed empirical evaluations to understand the impact of Lipschitz continuity of transition models, specifically when the transition model is used to perform multi-step state-predictions and policy improvements. We chose two standard domains: Cart Pole and Pendulum. In Cart Pole, we trained a network on a dataset of $15*10^3$ tuples $\langle s,a,s' \rangle$. During training, we ensured that the weights of the network are smaller than $k$. For each $k$, we performed 20 independent model estimation, and chose the model with median cross-validation error. 

Using the learned model, along with the actual reward signal of the environment, we then performed stochastic actor-critic RL. \citep{barto1983neuronlike,sutton2000policy} This required an interaction between the policy and the learned model for relatively long trajectories. To measure the usefulness of the model, we then tested the learned policy on the actual domain. We repeated this experiment on Pendulum. To train the neural transition model for this domain we used $10^4$ samples. Notably, we used deterministic policy gradient \citep{silver2014deterministic} for training the policy network with the hyper parameters suggested by \citet{lillicrap2015continuous}. We report these results in Figure \ref{fig:cart_and_pendulum}. 

Observe that an intermediate Lipschitz constant yields the best result. Consistent with the theory, controlling the Lipschitz constant in practice can combat the compounding errors and can help in the value estimation problem. This ultimately results in learning a better policy.
\begin{figure}
\centering
        \begin{subfigure}
        \centering
                \includegraphics[width=0.425\columnwidth]{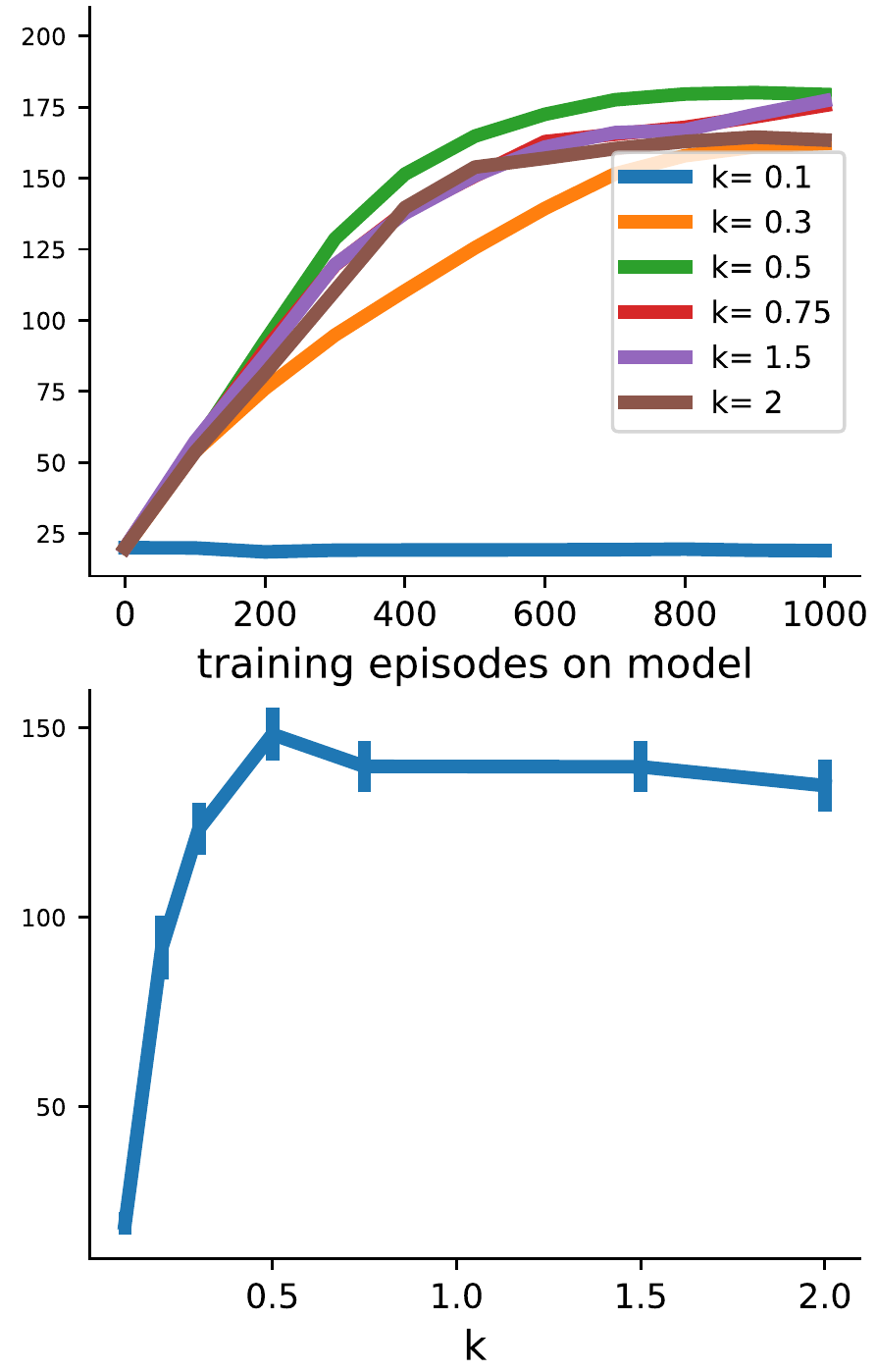}
        \end{subfigure}%
        \hspace{-.25cm}
        \begin{subfigure}
        \centering
                \includegraphics[width=.5\columnwidth]{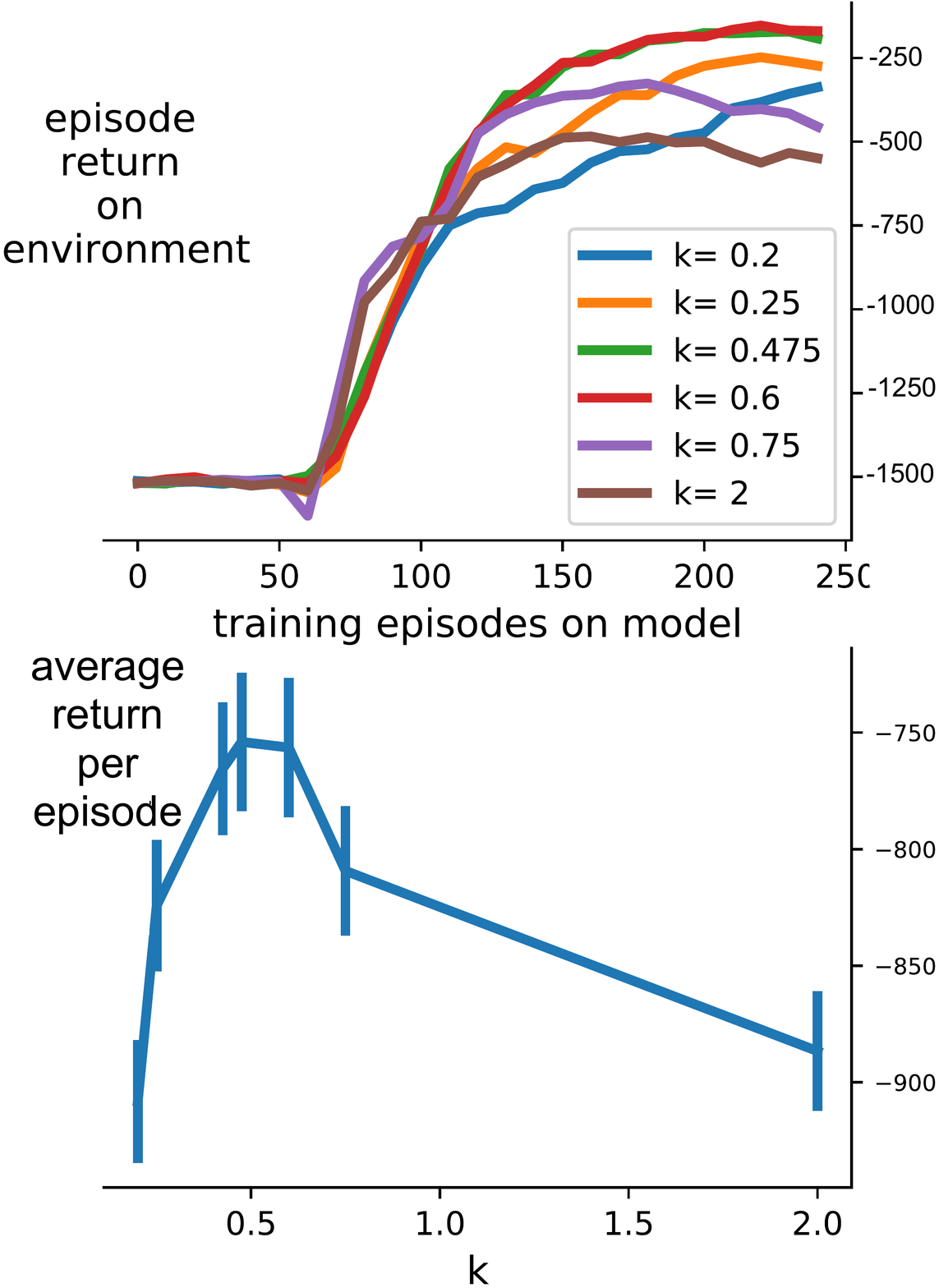}
        \end{subfigure}
        \caption{Impact of Lipschitz constant of learned models in Cart Pole (left) and Pendulum (right). An intermediate value of $k$ (Lipschitz constant) yields the best performance.}
        \label{fig:cart_and_pendulum}
\end{figure}

We next examined if the benefits carry over to stochastic settings. To capture stochasticity we need an algorithm to learn a Lipschitz model class (Definition \ref{Lipschitz_model_class}). We used an EM algorithm to joinly learn a set of functions $f$, parameterized by $\theta\!=\!\{\theta^{f}\!:\!f\in F_{g}\}$, and a distribution over functions $g$. Note that in practice our dataset only consists of a set of samples $\langle s,a,s'\rangle $ and does not include the function the sample is drawn from. Hence, we consider this as our latent variable $z$. As is standard with EM, we start with the log-likelihood objective (for simplicity of presentation we assume a single action in the derivation):
\begin{eqnarray*}
L(\theta)&=&\sum_{i=1}^{N}\log p(s_i,{s_{i}}';\theta)\\
&=&\!\sum_{i=1}^{N}\log \sum_{f} p(z_i=f,s_i,{s_{i}}';\theta)\\
&=&\!\sum_{i=1}^{N}\log \sum_{f}\!q(z_i\!=\!f|s_i,{s_{i}}')\!\frac{p(z_i=f,s_i,{s_{i}}';\theta)}{q(z_i=f|s_i,{s_{i}}')} \\
&\geq&\!\sum_{i=1}^{N}\!\sum_{f}\!q(z_i\!=\!f|s_i,{s_{i}}')\!\log\frac{p(z_i=f,s_i,{s_{i}}';\theta)}{q(z_i=f|s_i,{s_{i}}')}\ ,\\
\end{eqnarray*}
where we used Jensen's inequality and concavity of $\textrm{log}$ in the last line. This derivation leads to the following EM algorithm. 

In the M step, find $\theta_t$ by solving for:
\begin{equation*}
    \!\textrm{arg}\!\max_{\theta}\!\sum_{i=1}^{N}\!\sum_{f}\!q_{t-1}(\!z_i=f|s_i,{s_{i}}')\!\log\!\frac{p(z_i=f,s_i,{s_{i}}';\theta)}{q_{t-1}(\!z_i=f|s_i,{s_{i}}')}
\end{equation*}
In the E step, compute posteriors:
\begin{equation*}
    q_{t}(\!z_i\!=\!f|s_i,{s_{i}}')\!=\!\frac{p(s_i,{s_{i}}'|z_i=f;{\theta_t}^{f})g(z_i=f;\theta_t)}{\sum_{f} p(s_i,{s_{i}}'|z_i=f;{\theta_t}^{f})g(z_i=f;\theta_t)}\ .
\end{equation*}
Note that we assume each point is drawn from a neural network $f$ with probability:
$$p\big(s_i,{s_{i}}'|z_i=f;{\theta_t}^{f}\big)=\mathcal{N}\Big(\big|{s_{i}}'-f(s_i,{\theta_t}^{f})\big|,\sigma^2\Big)\ ,$$and with a fixed variance $\sigma^2$ tuned as a hyper-parameter.

We used a supervised-learning domain to evaluate the EM algorithm. We generated 30 points from 5 functions (written at the end of Appendix) and trained 5 neural networks to fit these points. Iterations of a single run is shown in Figure \ref{fig:EM_evolution} and the summary of results is presented in Figure \ref{fig:deep_em}. Observe that the EM algorithm is effective, and that controlling the Lipschitz constant is again useful.
\begin{figure}
\centering
        \begin{subfigure}
        \centering
                \includegraphics[width=0.45\columnwidth]{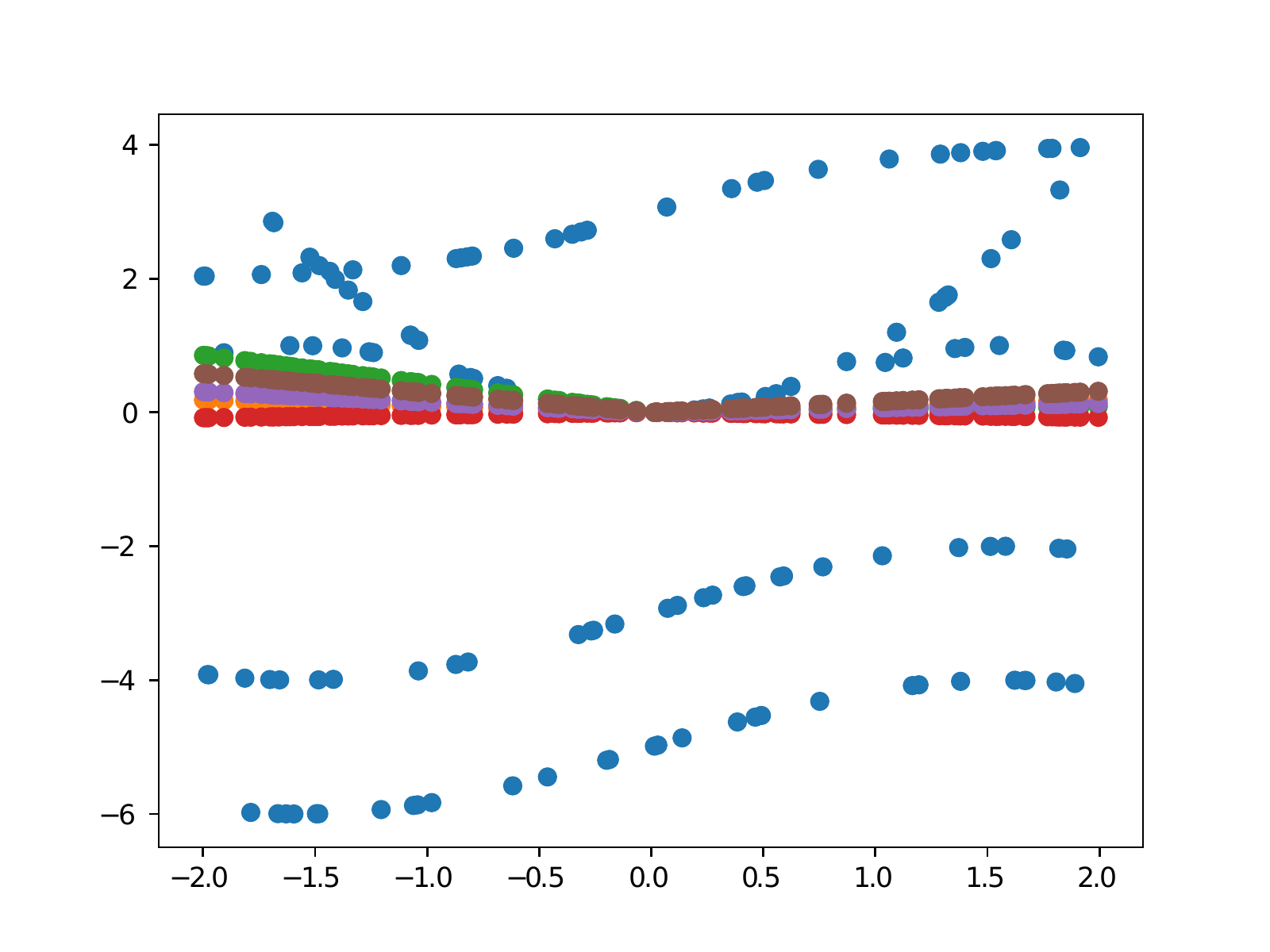}
        \end{subfigure}%
        \begin{subfigure}
        \centering
                \includegraphics[width=.45\columnwidth]{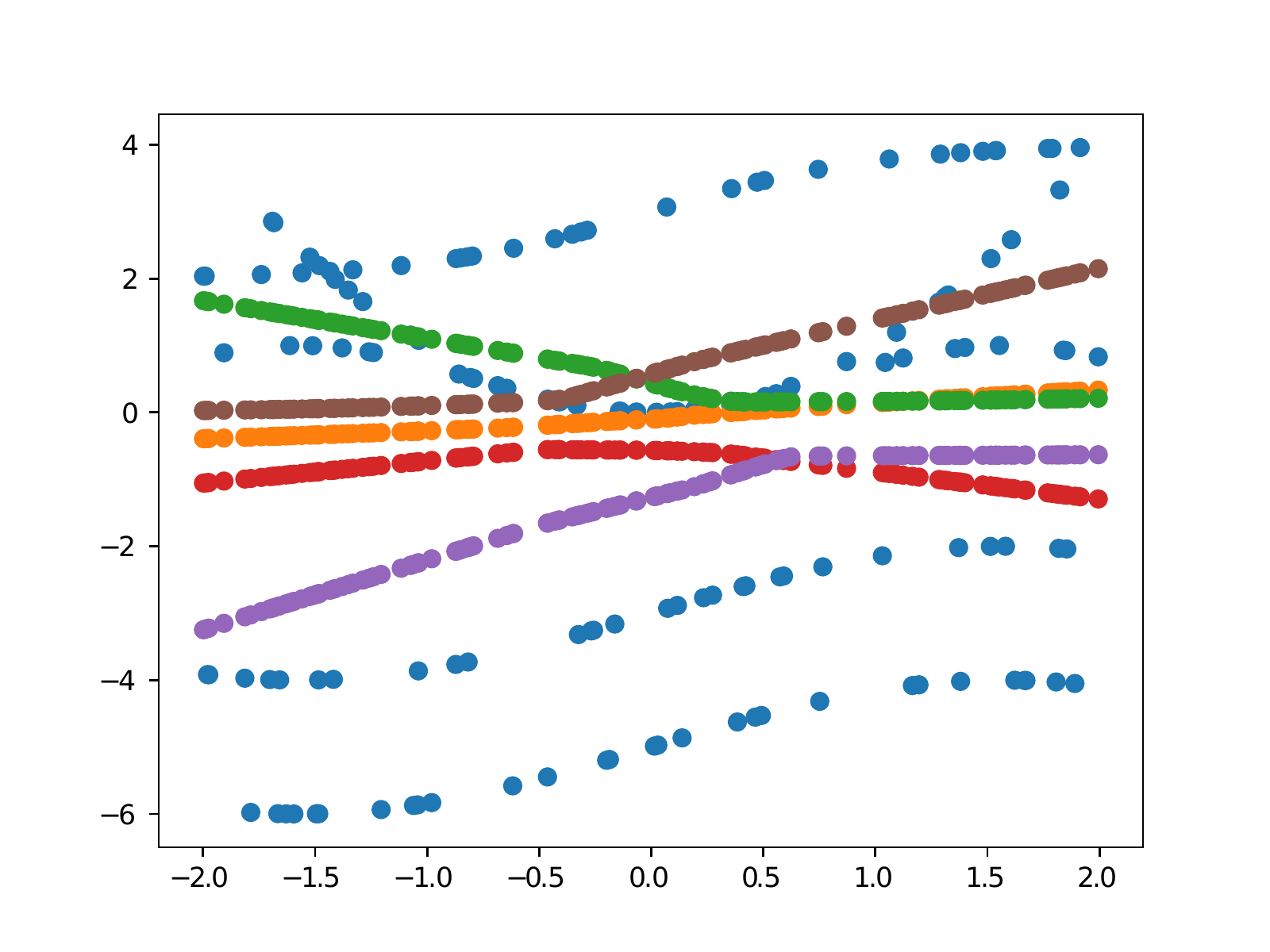}
        \end{subfigure}
        \begin{subfigure}
        \centering
                \includegraphics[width=.45\columnwidth]{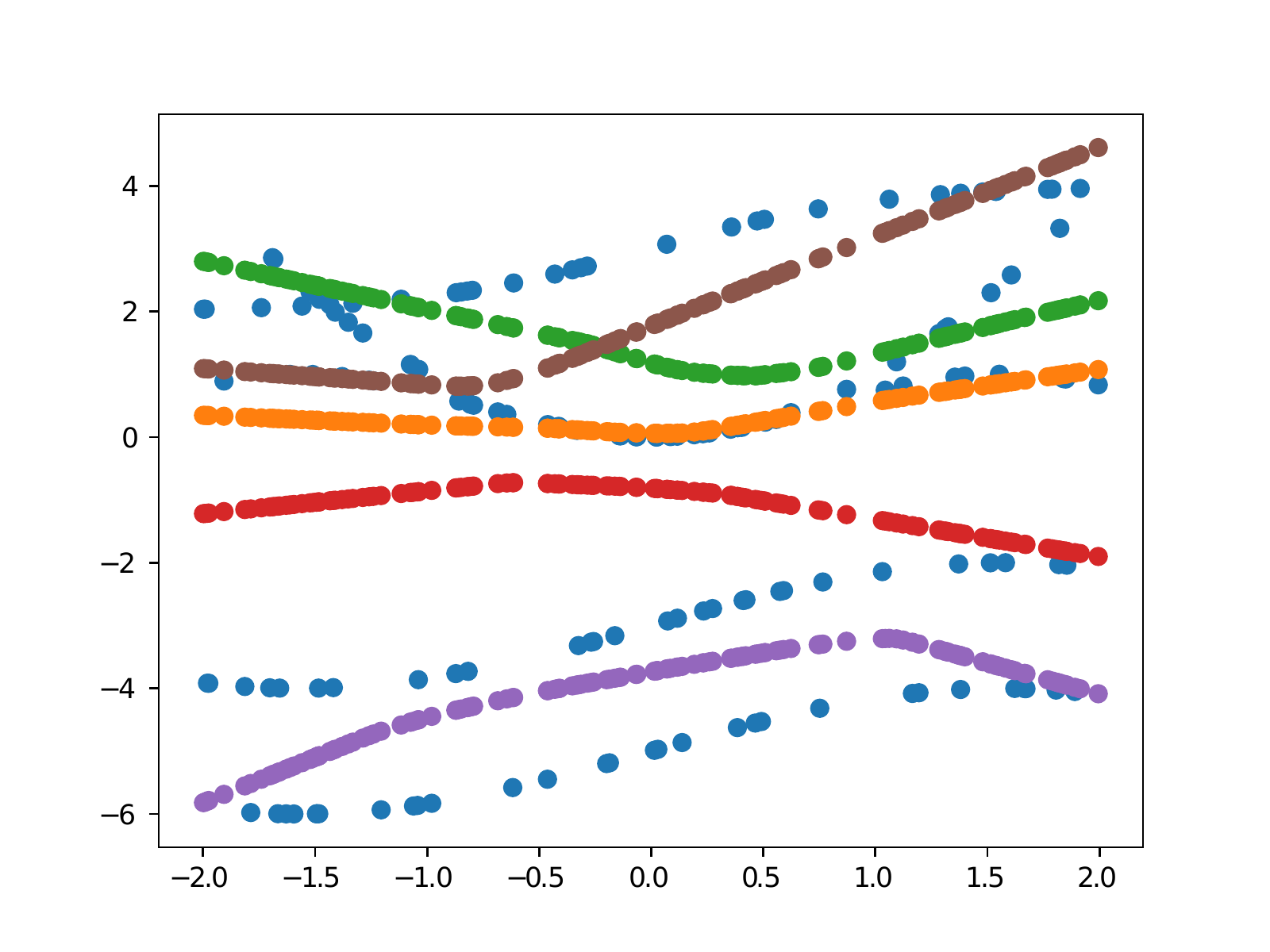}
        \end{subfigure}
        \begin{subfigure}
        \centering
                \includegraphics[width=0.45\columnwidth]{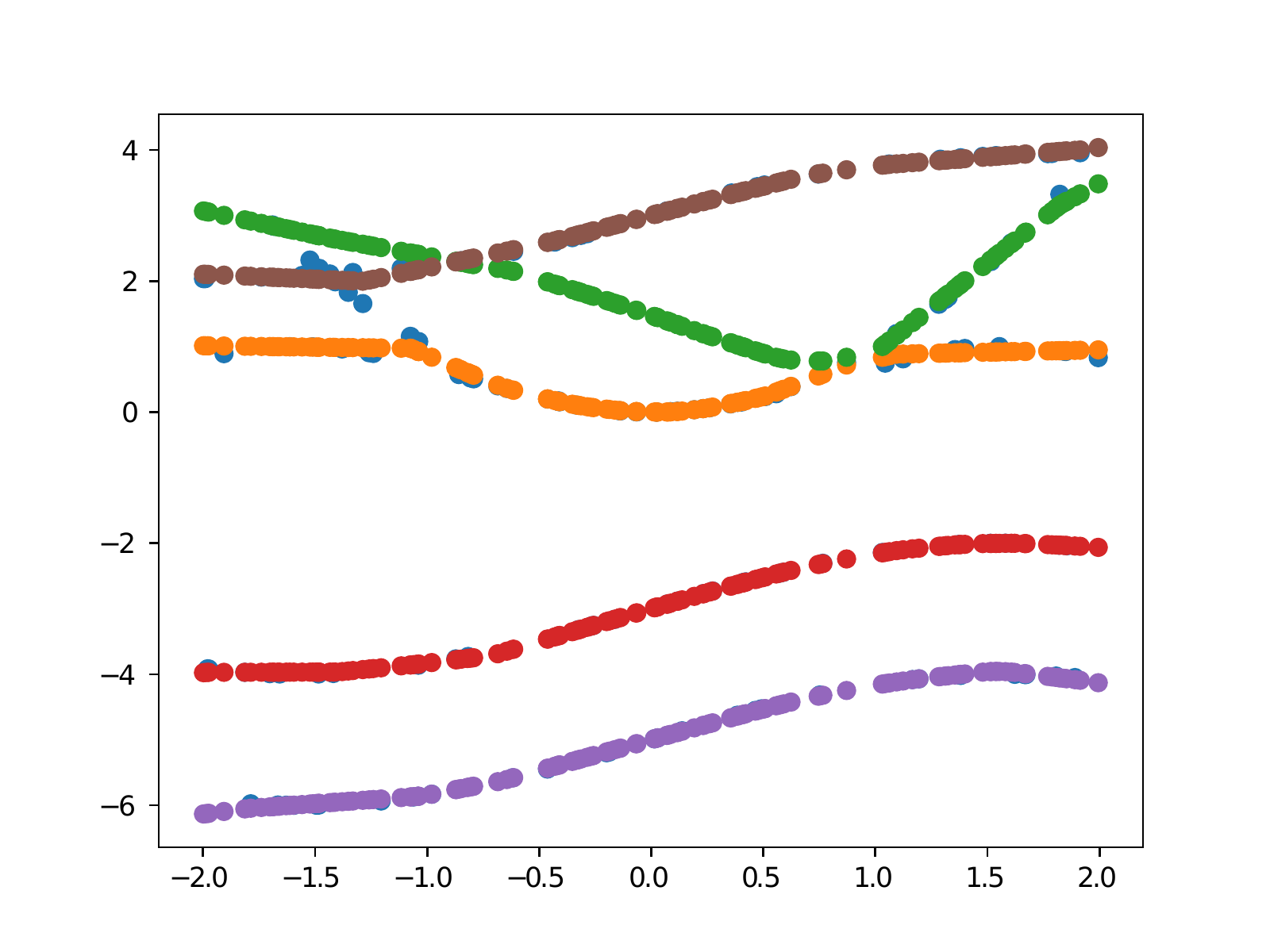}
        \end{subfigure}
        \caption{A stochastic problem solved by training a Lipschitz model class using EM. The top left figure shows the functions before any training (iteration 0), and the bottom right figure shows the final results~(iteration 50).}
        \label{fig:EM_evolution}
\end{figure}
\begin{figure}
        \centering
                \includegraphics[width=0.95\columnwidth]{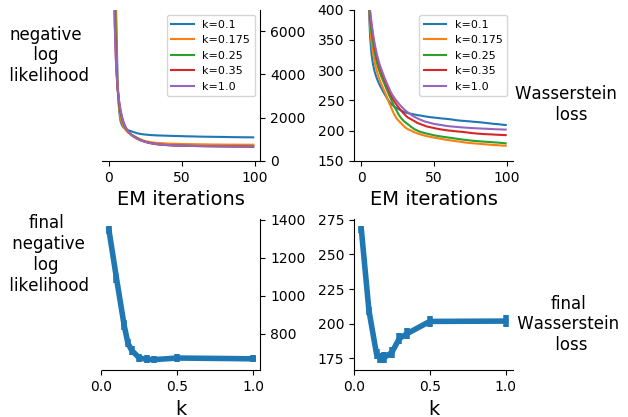}
        \label{fig:deep_em}
        \caption{Impact of controlling the Lipschitz constant in the supervised-learning domain. Notice the U-shape of final Wasserstein loss with respect to Lipschitz constant $k$.}
\end{figure}

We next applied EM to train a transition model for an RL setting, namely the gridworld domain from \citet{moerland2017learning}. Here a useful model needs to capture the stochastic behavior of the two ghosts. We modify the reward to be -1 whenever the agent is in the same cell as either one of the ghosts and 0 otherwise. We performed environmental interactions for 1000 time-steps and measured the return. We compared against standard tabular methods\citep{sutton98}, and a deterministic model that predicts expected next state \citep{linearDyna,parr2008analysis}. In all cases we used value iteration for planning.

Results in Figure \ref{fig:pacman} show that tabular models fail due to no generalization, and expected models fail since the ghosts do not move on expectation, a prediction not useful for planner. Performing value iteration with a Lipschitz model class outperforms the baselines.
\begin{figure}
\centering
        \begin{subfigure}
        \centering
                \includegraphics[width=0.5\columnwidth]{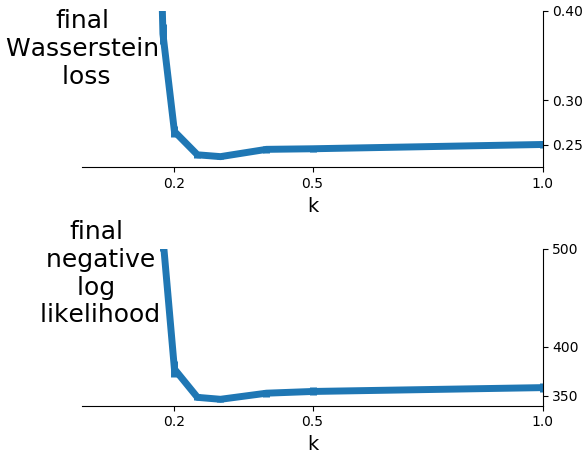}
        \end{subfigure}%
        \hspace{-.35cm}
        \begin{subfigure}
        \centering
                \includegraphics[width=.45\columnwidth]{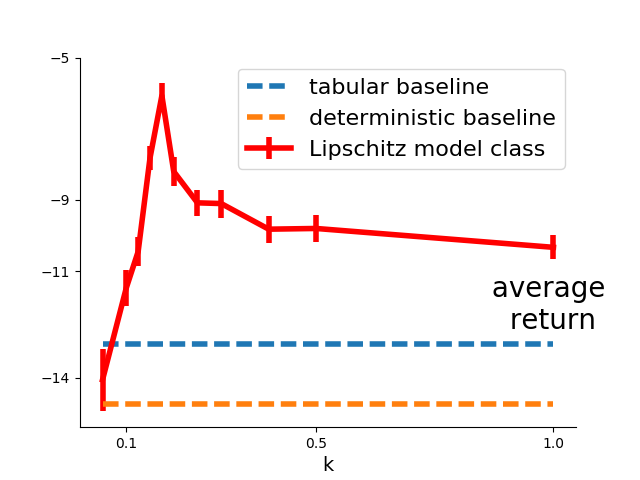}
        \end{subfigure}
        \caption{Performance of a Lipschitz model class on the gridworld domain. We show model test accuracy (left) and quality of the policy found using the model (right). Notice the poor performance of tabular and expected models.}
        \label{fig:pacman}
\end{figure}
\section{Conclusion}
We took an important step towards understanding model-based RL with function approximation. We showed that Lipschitz continuity of an estimated model plays a central role in multi-step prediction error, and in value-estimation error. We also showed the benefits of employing Wasserstein for model-based RL. An important future work is to apply these ideas to larger problems. 
\section{Acknowledgements}
The authors recognize the assistance of Eli Upfal, John Langford, George Konidaris, and members of Brown's Rlab specifically Cameron Allen, David Abel, and Evan Cater. The authors also thank anonymous ICML reviewer 1 for insights on a value-aware interpretation of Wasserstein. 
\bibliography{Lipschitz_paper_refs}
\bibliographystyle{icml2018}

\onecolumn
\appendix
\section*{Appendix}
\begin{claim}\label{claim-finite-mdp}In a finite MDP, transition probabilities can be expressed using a finite set of deterministic functions and a distribution over the functions.
\end{claim}
\begin{proof}
Let $Pr(s,a,s')$ denote the probability of a transiton from $s$ to $s'$ when executing the action $a$. Define an ordering over states $s_1, ... , s_n$ with an additional unreachable state $s_0$. Now define the cumulative probability distribution:
$$C(s,a,s_i):=\sum_{j=0}^i Pr(s,a,s_j)\ .$$
Further define $L$ as the set of distinct entries in $C$:
$$L:=\Big\{C(s,a,s_i) |\quad s \in \fancyS, i \in [0,n]\Big\}\ .$$
Note that, since the MDP is assumed to be finite, then $|L|$ is finite. We sort the values of $L$ and denote, by $c_i$, $i$th smallest value of the set. Note that $c_0=0$ and $c_{|L|}=1$. We now build determinstic set of functions $f_1, ..., f_{|L|}$ as follows:
$\forall i=1\ \textrm{to}\ |L|$ and $\forall j=1\ \textrm{to}\ n$, define $f_{i}(s) = s_j$ if and only if:
  $$C(s,a,s_{j-1}) < c_i \leq C(s,a,s_j)\ .$$ We also define the probability distribution $g$ over $f$ as follows:
  $$g(f_i|a):=c_i-c_{i-1}\ .$$
  Given the functions $f_1, ..., f_{|L|}$ and the distribution $g$, we can now compute the probability of a transition to $s_j$ from $s$ after executing action $a$:
\begin{eqnarray*}
    &&\sum_{i} \mathds{1}(f_{i}(s) = s_j)\ g(f_{i}|a)\\
&&= \sum_{i}\mathds{1}\big(C(s,a,s_{j-1}) < c_i \leq C(s,a,s_j)\big)\ (c_i-c_{i-1})\\
&&=  C(s,a,s_j) - C(s,a,s_{j-1})\\
&&= Pr(s,a,s_j)\ ,
\end{eqnarray*}
where $\mathds{1}$ is a binary function that outputs one if and only if its condition holds. We reconstructed the transition probabilities using distribution $g$ and deterministic functions $f_1, ..., f_{|L|}$.
\end{proof}
\begin{claim}\label{claim-deterministic}Given a deterministic and linear transition model, and a linear reward signal, the bounds provided in Theorems \ref{theorem_compound} and \ref{theorem:value_error} are both tight.
\end{claim}
Assume a linear transition function $T$ defined as:
$$T(s)=K s$$
Assume our learned transition function $\hat{T}$:
$$\hat T(s):=K s+\Delta$$
Note that: $$\max_{s}\big|T(s)-\hat T(s)\big|=\Delta$$
and that:
$$\min\{K_T,K_{\hat T}\}=K$$
First observe that the bound in Theorem 2 is tight for $n=2$:
\begin{eqnarray*}
	\forall s\quad \Big|T\big(T(s)\big)-\hat T\big(\hat T(s)\big)\Big|=\Big|K^2 s-K^2 s +\Delta(1+K)\Big|=\Delta \sum_{i=0}^{1}K^{i}
\end{eqnarray*}
and more generally and after $n$ compositions of the models, denoted by $T^{n}$ and $\hat{T}^{n}$, the following equality holds:
\begin{eqnarray*}
	\forall s\quad \Big|T^{n}(s)-\hat T^{n}(s)\Big|=\Delta \sum_{i=0}^{n-1}K^{i}
\end{eqnarray*}
Lets further assume that the reward is linear:
$$R(s)=K_{R}s$$
Consider the state $s=0$. Note that clearly $v(0)=0$. We now compute the value predicted using $\hat T$, denoted by $\hat v(0)$:
\begin{eqnarray*}
	\hat v(0)&=&R(0)+\gamma R(0+\Delta \sum_{i=0}^{0}K^{i})+\gamma^{2}R(0+\Delta \sum_{i=0}^{1}K^{i})+\gamma^{3}R(0+\Delta \sum_{i=0}^{2}K^{i})+...\\
	&=& 0+\gamma K_{R}\Delta \sum_{i=0}^{0}K^{i}+\gamma^{2}K_{R}\Delta \sum_{i=0}^{1}K^{i})+\gamma^{3}K_{R}\Delta \sum_{i=0}^{2}K^{i}+...\\
	&=&\gamma K_{R}\Delta \sum_{n=0}^{\infty}\gamma^n \sum_{i=0}^{n-1}K^{i}=\frac{\gamma K_R\Delta}{(1-\gamma)(1-\gamma \bar{K})}\ ,
\end{eqnarray*}
and so:
$$|v(0)-\hat v(0)|=\frac{\gamma K_R\Delta}{(1-\gamma)(1-\gamma \bar{K})}$$
Note that this exactly matches the bound derived in our Theorem \ref{theorem:value_error}.
\primelemma*
\begin{proof}

\begin{eqnarray*}
W\big(\widehat T(\cdot\mid \mu_1,a),\widehat T(\cdot\mid \mu_2,a)\big)&:=&\inf_j\int_{s'_1}\int_{s'_2}j(s'_1,s'_2)d(s'_1,s'_2)ds'_1 ds'_2\\
&=&\inf_j\int_{s_1}\int_{s_2}\int_{s'_1}\int_{s'_2}\sum_{f} \mathds{1}\big(f(s_1)=s'_1\wedge f(s_2)=s'_2\big)  j(s_1,s_2,f)d(s'_1,s'_2)ds'_1 ds'_2 ds_1 ds_2\\
&=&\inf_j\int_{s_1}\int_{s_2}\sum_{f}  j(s_1,s_2,f)d\big(f(s_1),f(s_2)\big)ds_1 ds_2\\
&\leq& K_F \inf_j\int_{s_1}\int_{s_2}\sum_{f}  g(f|a)j(s_1,s_2)d(s_1,s_2)ds_1 ds_2\\
&=& K_F\sum_{f}  g(f|a) \inf_j\int_{s_1}\int_{s_2}j(s_1,s_2)d(s_1,s_2)ds_1 ds_2\\
&=&K_F\sum_{f}  g(f|a)W(\mu_1,\mu_2)=K_FW(\mu_1,\mu_2)
\end{eqnarray*}

Dividing by $W(\mu_1,\mu_2)$ and taking $\sup$ over $a,\mu_1,$ and $\mu_2$, we conclude:
$$K^{\fancyA}_{W,W}(\widehat{T})=\sup_{a}\sup_{\mu_1,\mu_2}\frac{W\big(\widehat T(\cdot\mid \mu_1,a),\widehat T(\cdot\mid \mu_2,a)\big)}{W(\mu_1,\mu_2)}\leq K_{F}\ .$$

We can also prove this using the Kantarovich-Rubinstein duality theorem:

For every $\mu_1, \mu_2$, and $a \in \fancyA$ we have:
\begin{eqnarray*}
  W\big(\ThatGn(\cdot\mid \mu_1,a),\ThatGn(\cdot\mid \mu_2,a)\big) &=& \sup_{f:K_{d_\fancyS,\real}(f) \le 1} \int_s \big( \ThatGn(s|\mu_1,a) - \ThatGn(s|\mu_2,a)\big) f(s)ds\\
 &=& \sup_{f:K_{d_\fancyS,\real}(f) \le 1} \int_s \int_{s_0}\Big(  \widehat{T}(s|s_0, a) \mu_1(s_0) - \widehat{T}(s|s_0, a) \mu_2(s_0)\Big) f(s)ds ds_0\\
 &=& \sup_{f:K_{d_\fancyS,\real}(f) \le 1} \int_s \int_{s_0}\widehat{T}(s|s_0, a) \Big(\mu_1(s_0) - \mu_2(s_0)\Big) f(s)ds ds_0\\
  &=& \sup_{f:K_{d_\fancyS,\real}(f) \le 1} \int_s \int_{s_0}\sum_t g(t\mid a) \mathbb{1}\big(t(s_0) = s\big) \Big(\mu_1(s_0) - \mu_2(s_0)\Big) f(s)dsds_0\\
  &=& \sup_{f:K_{d_\fancyS,\real}(f) \le 1}  \sum_{t} g(t\mid a) \int_{s_0} \int_{s} \mathbb{1}\big(t(s_0) = s\big) \big(\mu_1(s_0) - \mu_2(s_0)\big) f(s)dsds_0\\
  &=& \sup_{f:K_{d_\fancyS,\real}(f) \le 1}   \sum_{t} g(t\mid a)  \int_{s_0} \big(\mu_1(s_0) - \mu_2(s_0)\big) f\big(t(s_0)\big)ds_0\\
    &\leq&  \sum_{t} g(t\mid a) \sup_{f:K_{d_\fancyS,\real}(f) \le 1}\int_{s_0} \big(\mu_1(s_0) - \mu_2(s_0)\big) f\big(t(s_0)\big) ds_0\\
    && \textrm{composition of $f, t$ is Lipschitz with constant upper bounded by $K_F$.}\\
    &=&  K_F\sum_{t} g(t\mid a) \sup_{f:K_{d_\fancyS,\real}(f) \le 1}\int_{s_0}\big(\mu_1(s_0) - \mu_2(s_0)\big) \frac{f(t(s_0))}{K_F}ds_0\quad \\
    &\leq&  K_F\sum_{t} g(t\mid a) \sup_{h:K_{d_\fancyS,\real}(h) \le 1}\int_{s_0}\big(\mu_1(s_0) - \mu_2(s_0)\big) h(s_0))ds_0\quad \\
    &=&  K_F \sum_{t} g(t\mid a) W(\mu_1, \mu_2) = K_F W(\mu_1, \mu_2)
\end{eqnarray*}
Again we conclude by dividing by $W(\mu_1,\mu_2)$ and taking $\sup$ over $a,\mu_1,$ and $\mu_2$.
\end{proof}

\compositionlemma*
\begin{proof}
	\begin{eqnarray*}
			K_{d_1,d_3}(h)&=&\sup_{s1,s_2}\frac{d_{3}\Big(f\big(g(s_1)\big),f\big(g(s_2)\big)\Big)}{d_{1}(s_1,s_2)}\\
			&=&\sup_{s_1,s_2}\frac{d_{2}\big(g(s_1),g(s_2)\big)}{d_{1}(s_1,s_2)}\frac{d_{3}\Big(f\big(g(s_1)\big),f\big(g(s_2)\big)\Big)}{d_{2}\big(g(s_1),g(s_2)\big)}\\
			&\leq&\sup_{s_1,s_2}\frac{d_{2}\big(g(s_1),g(s_2)\big)}{d_{1}(s_1,s_2)}
			\sup_{s_1,s_2}\frac{d_{3}\big(f(s_1),f(s_2)\big)}{d_{2}(s_1,s_2)}\\
			&=&K_{d_1,d_2}(g)K_{d_2,d_3}(f).
	\end{eqnarray*}
\end{proof}
\mullerLemma*
\begin{proof}
	\begin{eqnarray*}
	K_{d_\fancyS,d_{\real}}^\fancyA \Big(\int_{s'}\widehat T(s'|s,a)f(s')ds'\Big)&=&\sup_{a}\sup_{s_1,s_2}\frac{|\int_{s'} \big( \widehat T(s'|s_1,a)-\widehat T(s'|s_2,a)\big)f(s')ds'|}{d(s_1,s_2)}\\
	&=&\sup_{a}\sup_{s_1,s_2}\frac{|\int_{s'} \big(\widehat T(s'|s_1,a)-\widehat T(s'|s_2,a)\big)f(s')\frac{K_{d_
	\fancyS,d_\real}(f)}{K_{d_
	\fancyS,d_\real}(f)}ds'|}{d(s_1,s_2)}\\
	&=&K_{d_
	\fancyS,d_\real}(f)\sup_{a}\sup_{s_1,s_2}\frac{|\int_{s'} \big( \widehat T(s'|s_1,a)-\widehat T(s'|s_2,a)\big)\frac{f(s')}{K_{d_
	\fancyS,d_\real}(f)}ds'|}{d(s_1,s_2)}\\
	&\leq&K_{d_
	\fancyS,d_\real}(f)\sup_{a}\sup_{s_1,s_2}\frac{|\sup_{g:K_{d_\fancyS,d_\real}(g)\leq 1}\int_{s'} \big( \widehat T(s'|s_1,a)-\widehat T(s'|s_2,a)\big)g(s')ds'|}{d(s_1,s_2)}\\
	&=&K_{d_
	\fancyS,d_\real}(f)\sup_{a}\sup_{s_1,s_2}\frac{\sup_{g:K_{d_\fancyS,d_\real}(g)\leq 1}\int_{s'} \big( \widehat T(s'|s_1,a)-\widehat T(s'|s_2,a)\big)g(s')ds'}{d(s_1,s_2)}\\
	&=&K_{d_
	\fancyS,d_\real}(f)\sup_{a}\sup_{s_1,s_2}\frac{W\big(\widehat T(\cdot|s_1,a),\widehat T(\cdot|s_2,a)\big)}{d(s_1,s_2)}\\
	&=& K_{d_
	\fancyS,d_\real}(f) K_{d_{\fancyS},W}^{\fancyA}(\widehat T) \ .
	\end{eqnarray*}
\end{proof}
\LipschitzOperators*
\begin{proof}
	1 was proven by \citet{littman1996generalized}, and 2 is proven several times~\cite{fox2016g,mellowmax,nachum2017bridging,neu2017unified}. We focus on proving 3. Define $$\rho(x)_i=\frac{e^{\beta x_i}}{\sum_{i=1}^n e^{\beta x_i}}\ ,$$ and observe that $boltz_{\beta}(x)=x^{\top}\rho(x)$. \citet{gao2017properties} showed that $\rho$ is Lipschitz:
	\begin{equation}\norm{\rho(x_1)-\rho(x_2)}_{2}\leq \beta \norm{x_1-x_2}_{2}\label{gao}\end{equation}
	Using their result, we can further show:
	\begin{eqnarray*}
		&&|\rho(x_1)^{\top}x_1 - \rho(x_2)^{\top}x_2|\\
		&&\leq|\rho(x_1)^{\top}x_1-\rho(x_1)^{\top}x_2| + |\rho(x_1)^{\top}x_2- \rho(x_2)^{\top}x_2|\\
		&&\leq \norm{\rho(x_1)}_{2}\norm{x_1-x_2}_{2}\\
		&&+\norm{x_2}_{2}\norm{\rho(x_1)-\rho(x_2)}_{2}\quad \textrm{(Cauchy-Shwartz)}\\
		&&\leq\norm{\rho(x_1)}_{2}\norm{x_1-x_2}_{2}\\
		&&+\norm{x_2}_{2}\beta\norm{x_1-x_2}_{2}\quad \textrm{\big( from Eqn~\ref{gao})\big)}\\
		&&\leq(1+\beta V_{\max}\sqrt{|A|})\norm{x_1-x_2}_{2}\\
		&&\leq(\sqrt{|A|}+\beta V_{\max}|A|)\norm{x_1-x_2}_{\infty}\ ,
	\end{eqnarray*}
dividing both sides by $\norm{x_1-x_2}_{\infty}$ leads to 3.
\end{proof}
Below, we derive the Lipschitz constant for various functions mentioned in Table~\ref{tab:lipschitz-neural}.\\
\textbf{ReLu non-linearity} We show that $\ReLu: \mathbb{R}^n \rightarrow \mathbb{R}^n$ has Lipschitz constant 1 for $p$. 

\begin{eqnarray*} 
K_{\|.\|_p, \|.\|_p}(\ReLu) &=& \sup_{x_1, x_2} \frac{\|\ReLu(x_1) - \ReLu(x_2)\|_p}{\|x_1 - x_2\|_p}\\
&=& \sup_{x_1, x_2} \frac{(\sum_i |\ReLu(x_1)_i - \ReLu(x_2)_i|^p)^{\frac{1}{p}}}{\|x_1 - x_2\|_p}\\
&& (\textrm{We can show that $|\ReLu(x_1)_i - \ReLu(x_2)_i| \le |x_{1,i} - x_{2,i}|$} \textrm{ and so}):\\
&\le& \sup_{x_1, x_2} \frac{ (\sum_i |x_{1,i} - x_{2,i}|^p)^{\frac{1}{p}}}{\|x_1 - x_2\|_p}\\
&=& \sup_{x_1, x_2} \frac{\|x_1 - x_2\|_p}{\|x_1 - x_2\|_p} = 1 \\
\end{eqnarray*}

\textbf{Matrix multiplication} Let $W \in \mathbb{R}^{n\times m}$. We derive the Lipschitz continuity for the function $\times W (x) = Wx$.

For $p=\infty$ we have:
\begin{eqnarray*}
	&&K_{\norm{}_\infty,\norm{}_\infty}\big(\times W( x_1)\big)\\
	&&=\sup_{x_1,x_2}\frac{\norm{\times W(x_{1})-\times W(x_{2})}_{\infty}}{\norm{x_{1}-x_{2}}_{\infty}}=\sup_{x_1,x_2}\frac{\norm{Wx_1-Wx_2}_{\infty}}{\norm{x_{1}-x_{2}}_{\infty}}
     = \sup_{x_1,x_2}\frac{\norm{W(x_{1}-x_{2})}_{\infty}}{\norm{x_{1}-x_{2}}_{\infty}}\\
	&&=\sup_{x_1,x_2}\frac{\sup_{j}|W_{j}(x_{1}-x_{2})|}{\norm{x_{1}-x_{2}}_{\infty}}\ \\
	&&\leq\sup_{x_1,x_2}\frac{\sup_{j}\norm{W_{j}}\norm{x_{1}-x_{2}}_{\infty}}{\norm{x_{1}-x_{2}}_{\infty}}\quad \textrm{(H\"{o}lder's inequality)}\\
	&&=\sup_{j}\norm{W_{j}}_{1}\ ,
\end{eqnarray*}
where $W_j$ refers to $j$th row of the weight matrix $W$. Similarly, for $p=1$ we have:

\begin{eqnarray*}
	&&K_{\norm{}_1,\norm{}_1}\big(\times W( x_1)\big)\\
	&&=\sup_{x_1,x_2}\frac{\norm{\times W(x_{1})-\times W(x_{2})}_1}{\norm{x_{1}-x_{2}}_1}=\sup_{x_1,x_2}\frac{\norm{Wx_1-Wx_2}_1}{\norm{x_{1}-x_{2}}_1}
     = \sup_{x_1,x_2}\frac{\norm{W(x_{1}-x_{2})}_1}{\norm{x_{1}-x_{2}}_1}\\
	&&= \sup_{x_1,x_2}\frac{\sum_{j}|W_{j}(x_{1}-x_{2})|}{\norm{x_{1}-x_{2}}_{1}}\\
	&&\leq\sup_{x_1,x_2}\frac{\sum_{j}\norm{W_{j}}_{\infty}\norm{x_{1}-x_{2}}_{1}}{\norm{x_{1}-x_{2}}_{1}}=\sum_{j}\norm{W_{j}}_{\infty} \ ,
\end{eqnarray*}
and finally for $p=2$:
\begin{eqnarray*}
	&&K_{\norm{}_2,\norm{}_2}\big(\times W( x_1)\big)\\
	&&=\sup_{x_1,x_2}\frac{\norm{\times W(x_{1})-\times W(x_{2})}_2}{\norm{x_{1}-x_{2}}_2}=\sup_{x_1,x_2}\frac{\norm{Wx_1-Wx_2}_2}{\norm{x_{1}-x_{2}}_2}
     = \sup_{x_1,x_2}\frac{\norm{W(x_{1}-x_{2})}_2}{\norm{x_{1}-x_{2}}_2}\\
	&&=\sup_{x_1,x_2}\frac{\sqrt{\sum_{j}|W_{j}(x_{1}-x_{2})|^2}}{\norm{x_{1}-x_{2}}_{2}}\\
	&&\leq\sup_{x_1,x_2}\frac{\sqrt{\sum_{j}\norm{W_{j}}_{2}^{2}\norm{x_{1}-x_{2}}_{2}^{2}}}{\norm{x_{1}-x_{2}}_{2}}=\sqrt{\sum_{j}\norm{W_{j}}_{2}^{2}} \quad .
\end{eqnarray*}

\paragraph{Vector addition} We show that $+b: \mathbb{R}^n \rightarrow \mathbb{R}^n$ has Lipschitz constant 1 for $p=0, 1, \infty$ for all $b\in \mathbb{R}^n$.

\begin{eqnarray*} 
K_{\|.\|_p, \|.\|_p}(\ReLu) &=& \sup_{x_1, x_2} \frac{\|+b(x_1) - +b(x_2)\|_p}{\|x_1 - x_2\|_p}\\
&=& \sup_{x_1, x_2} \frac{\|(x_1 + b) - (x_2 + b)\|_p}{\|x_1 - x_2\|_p} = \frac{\|x_1 - x_2\|_p}{\|x_1 - x_2\|_p} = 1\\
\end{eqnarray*}
\paragraph{Supervised-learning domain}
We used the following 5 functions to generate the dataset:
\begin{eqnarray*}
f_0(x)&=&\textrm{tanh}(x)+3\\
f_1(x)&=&x*x\\
f_2(x)&=&\textrm{sin}(x)-5\\
f_3(x)&=&\textrm{sin}(x)-3\\
f_4(x)&=&\textrm{sin}(x)*\textrm{sin}(x)\\
\end{eqnarray*}
We sampled each function 30 times, where the input was chosen uniformly randomly from $[-2,2]$ each time.
\end{document}